\documentclass[11pt]{article}

\usepackage[final]{acl}

\usepackage{times}
\usepackage{latexsym}

\usepackage[T1]{fontenc}

\usepackage[utf8]{inputenc}

\usepackage{microtype}

\usepackage{inconsolata}

\usepackage{graphicx}

\usepackage{microtype}

\usepackage{inconsolata}
\usepackage{graphicx}
\usepackage{amsmath}
\usepackage{amssymb}
\usepackage{booktabs}
\usepackage{multirow}

\usepackage{algorithm}
\usepackage{algpseudocode}
\usepackage{tcolorbox}

\usepackage{amsmath}
\usepackage{amssymb}
\usepackage{booktabs,multirow}

\usepackage{enumitem}
\usepackage{adjustbox}

\usepackage{amsthm}
\newtheorem{lemma}{Lemma}
\newtheorem{theorem}{Theorem}
\usepackage{bbm} 

\usepackage{orcidlink}
\usepackage{inconsolata}
\usepackage{listings}
\usepackage{minted}

\usepackage{xcolor,colortbl}
\definecolor{fomulblue}{HTML}{EAF2FB}
\definecolor{fomuldeep}{HTML}{D7E8FA}
\definecolor{quantgray}{HTML}{F6F7F9}
\definecolor{softgray}{HTML}{F6F7F9}
\definecolor{warnred}{HTML}{FCECEC}
\definecolor{fomulgray}{HTML}{F4F6F8}
\definecolor{fomulgreen}{HTML}{EAF7EA}

\usepackage{silence}
\title{Forgetting Only What Matters: Layer-Selective Unlearning toward Robust LLMs}

\author{Ravi Ranjan\,\orcidlink{0009-0004-5790-3179} \thanks{Corresponding author.\\ Published as a conference paper at AACL-IJCNLP 2026.}\\
  Florida International\\
  University (FIU), \\
  Miami, FL, USA \\
  \href{mailto:rkuma031@fiu.edu}
  {rkuma031@fiu.edu}\\\And
  Olivera Kotevska\,\orcidlink{0000-0003-1677-2243}\\
  Oak Ridge National\\
  Laboratory (ORNL), \\
  Oak Ridge, TN, USA \\
  \href{mailto:kotevskao@ornl.gov}{kotevskao@ornl.gov}\\\And
  Agoritsa Polyzou\,\orcidlink{0000-0001-8630-7131}\\
  Florida International\\
  University (FIU), \\
  Miami, FL, USA \\
  \href{mailto:apolyzou@fiu.edu}{apolyzou@fiu.edu} }

\begin{document}
\maketitle

\vspace{-0.3cm}
\begin{abstract}
\vspace{-0.2cm}
Large Language Models (LLMs) can memorize and reproduce sensitive, copyrighted, or otherwise undesirable training content, creating privacy, safety, and regulatory concerns. Machine unlearning offers a practical alternative to full retraining, but many existing methods apply broad or fixed parameter updates that can degrade utility and remain brittle under deployment changes such as post-training quantization, where forgotten knowledge may partially re-emerge. 
We propose \emph{\textbf{F}orgetting \textbf{O}nly What \textbf{M}atters via \textbf{U}nlearning \textbf{L}ayers (\textbf{FOM-UL})}, a layer-level unlearning framework that selects transformer layers using a forget-to-retain significance score. This score identifies layers with high influence on the forget set and low sensitivity to the retain set, allowing FOM-UL to concentrate updates where they are most effective while leaving most of the model unchanged. This targeted update strategy improves the forgetting-utility trade-off and provides an empirical path toward quantization-resilient unlearning by reducing the chance that small, diffuse updates are erased by low-bit rounding. 
Across TOFU, KnowUnDo, and MUSE-style evaluations, FOM-UL reduces residual memorization compared with strong GA, NPO, KLD, SURE, ReLearn, and LUNAR-based baselines while preserving retain-set utility close to the vanilla model. Under 8-bit and 4-bit post-training quantization, FOM-UL maintains stronger memorization suppression and utility preservation than competing methods, and adversarial prompt evaluations show lower recovery of forgotten content. Overall, FOM-UL provides an efficient unlearning strategy that improves targeted forgetting, utility preservation, and deployment robustness without claiming formal guarantees of erasure.\\
Code available at:\\ \href{https://github.com/raviranjan-ai/FOMUL-AACL-2026}{https://github.com/raviranjan-ai/FOMUL-AACL-2026}.
\end{abstract}

\vspace{-0.3cm}
\section{Introduction}
\vspace{-0.4cm}
\begin{figure}[t!]
  \centering
  \includegraphics[width=\linewidth]{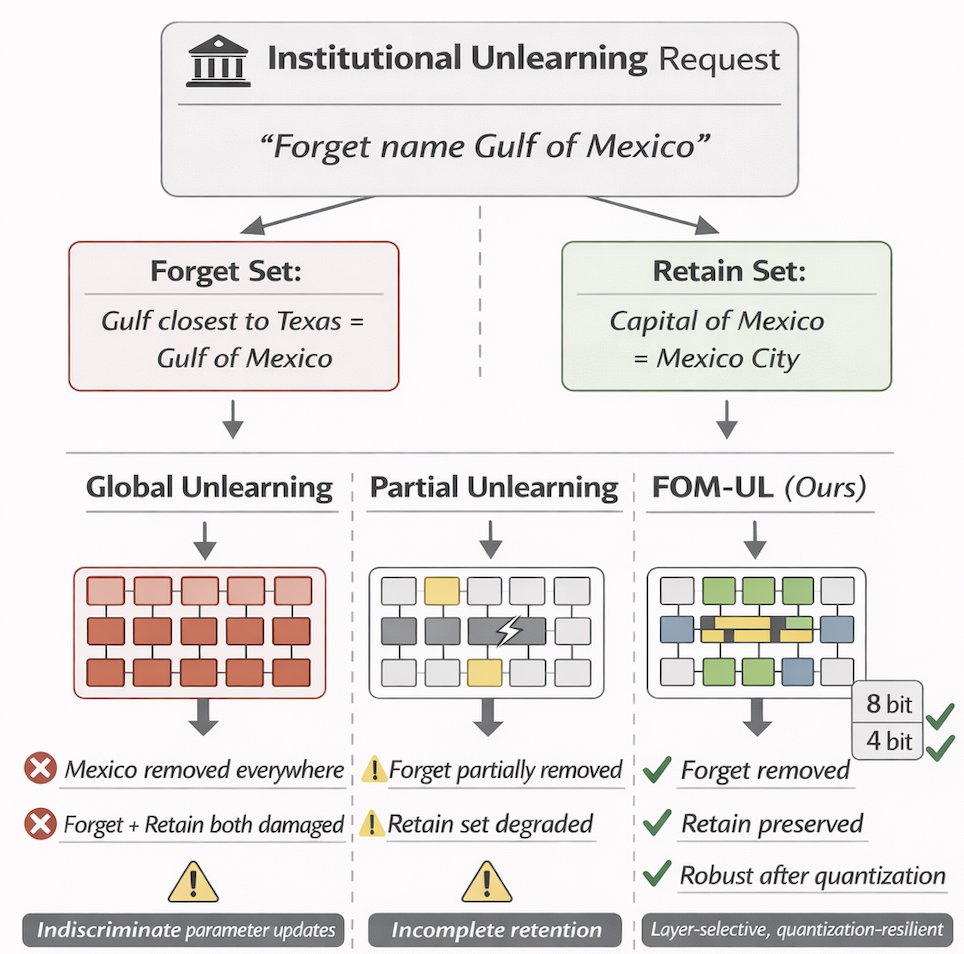}
  \caption{FOM-UL robust and quantization-resilient forgetting against global and partial unlearning methods.}
  \label{fig:intro}
\vspace{-0.6cm}
\end{figure}

Large Language Models (LLMs) have transformed natural language processing, delivering strong performance across diverse tasks and domains~\cite{zhao2023survey}. Yet, as these models scale and are deployed widely, an increasingly visible failure mode is their tendency to memorize and reproduce fragments of training data, including sensitive personal information, copyrighted text, or otherwise undesirable content~\cite{zhao2023survey,ahmed2026extracting}. Such memorization creates legal, ethical, and security risks, especially in high-stakes settings where accidental disclosure of protected content is unacceptable. These concerns are further amplified by regulatory requirements such as the General Data Protection Regulation (GDPR) ``right to be forgotten''~\cite{council2022general}, which demands mechanisms to remove specific data upon request.

Machine unlearning for LLMs has therefore emerged as a practical alternative to full retraining: the goal is to remove targeted knowledge or behaviors from a trained model while preserving its overall capabilities. However, existing unlearning pipelines face major obstacles. Retraining is often prohibitively expensive at LLM scale~\cite{jang2022knowledge}, and frequent unlearning requests in real-world deployments further stress the need for efficient, deployable solutions. Beyond cost, the high dimensionality and tightly coupled representations of modern transformer architectures make targeted knowledge removal inherently difficult: even seemingly localized edits can be propagated broadly, causing utility degradation or catastrophic forgetting~\cite{zhang2024negative}. These challenges have motivated a rapidly growing literature on LLM unlearning across widely used architectures such as GPT-2, LLaMA, and Gemma~\cite{geng2025comprehensive,yao2024machine}.

A central issue is that many unlearning methods still rely on global or otherwise indiscriminate parameter updates, which can be brittle and imprecise. Figure~\ref{fig:intro} contrasts this paradigm with our proposed approach: rather than updating large portions of the model, \textbf{Forgetting Only What Matters via Unlearning Layers (FOM-UL)} selectively modifies only those transformer layers most responsible (high forget-to-retain ratio) for encoding the undesired knowledge. This targeted intervention aims to improve approximate forgetting while preserving general capabilities and minimizing collateral damage.

\noindent \textit{Targeted unlearning} differs from general-purpose model editing by focusing on the selective removal of \emph{specific} facts, documents, or behaviors, enabling an LLM to ``forget'' designated content without retraining from scratch. Importantly, unlearning is not standard fine-tuning in reverse: whereas fine-tuning typically reinforces desired behavior via positive examples, unlearning must suppress undesired behavior.

A range of approaches have been proposed to realize this objective~\cite{geng2025comprehensive}. In particular, preference-optimization (PO) and gradient-ascent (GA) style methods are widely adopted due to their simplicity and effectiveness~\cite{liu2024large}. PO-based methods cast unlearning as an alignment problem in which the model is steered to prefer alternative responses using preference pairs.
Complementary directions, including relabeling, adapter/LoRA-based updates, quantization-based techniques~\cite{zhang2024does}, and reinforcement learning~\cite{lu2022quark}-provide additional tools, but PO and GA remain central to scalable unlearning.

\noindent Gradient Ascent (GA) is a common unlearning strategy that increases the loss on memorized responses to reduce confidence in targeted outputs~\cite{yao2024large}. However, GA typically applies broad model-wide updates, which can leave residual memorization and harm unrelated capabilities~\cite{yao2024machine}, and it is often brittle under post-training quantization, where discretizations may partially restore suppressed behaviors. Although retain-set regularization~\cite{liu2024large} and KL-divergence constraints to the original model~\cite{yao2024machine} help preserve utility, they do not fully eliminate the side effects of indiscriminate updates.
Motivated by recent evidence of quantization-induced failure modes in unlearning~\cite{zhang2024catastrophic}, we instead pursue a layer-selective strategy that concentrates updates where they matter most, improving precision and robustness against quantization-driven recovery.

\vspace{0.2cm}
\noindent \textbf{FOM-UL} is a targeted and efficient unlearning framework that mitigates catastrophic forgetting, quantization-induced relearning, and the inefficiency of global parameter updates.
Our primary contributions are:

\textbf{(i) Layer-level forget-retain localization.}
We introduce a layer-selection criterion based on the ratio between forget-set gradient magnitude and retain-set gradient magnitude. This identifies layers that offer high forgetting leverage with comparatively low retain-set interference.

\textbf{(ii) Iterative layer-budget expansion.}
Rather than updating a fixed region or the full model, FOM-UL starts from a small high-score layer set and expands it only when forgetting criteria are unmet, improving the trade-off between erasure strength and utility preservation.

\textbf{(iii) Efficient selective optimization.}
By restricting updates to a small number of selected layers, FOM-UL reduces trainable parameters, memory use, and runtime while remaining compatible with standard GA, NPO, and KLD-style unlearning losses.

\textbf{(iv) Empirical quantization-robustness analysis.}
Motivated by quantization-induced recovery, we test FOM-UL under 8-bit and 4-bit post-training quantization and show that concentrated layer updates reduce residual memorization compared with global or fixed-selection baselines.

\section{Preliminary and Related Work}
\label{sec:related-work}
\vspace{-0.2cm}
\textbf{Machine unlearning} in Large Language Models (LLMs) involves selectively removing specific learned knowledge without significantly degrading overall model performance \cite{geng2025comprehensive,liu2025rethinking,jang2022knowledge,huang2024offset}. Formally, given a model $f_\theta$ parameterized by $\theta$ and a dataset $D_{\text{forget}}$ representing undesirable knowledge, the parameter updates of machine unlearning can be expressed as:
\begin{equation}
\theta' = \theta + \eta \nabla_{\theta} \mathcal{L}_{\text{forget}}(\theta, D_{\text{forget}}),
\end{equation}
where $\eta$ is the learning rate and $\mathcal{L}_{\text{forget}}$ is typically a loss function defined on the forget set, often optimized via gradient ascent (GA) \cite{bourtoule2021machine,golatkar2020forgetting}. 

\noindent \textbf{Recent advances} in LLM unlearning have established several effective approaches to remove specific knowledge or capabilities from trained models. Parameter-based methods modify model weights directly through techniques like gradient ascent \cite{jang2022knowledge} or localized weight editing \cite{ilharco2022editing}. \textit{Global unlearning} updates the full parameter space of an LLM, which can aggressively suppress the targeted behavior but often propagates changes widely, leading to broader utility degradation and higher computational cost \cite{wuerkaixi2025adaptive}. In contrast, \textit{partial unlearning} restricts updates to a subset of components (e.g., layers, modules, or adapters) to limit collateral damage and improve efficiency \cite{wang2025llm}, yet it can suffer from \emph{incomplete forgetting} when the targeted knowledge is distributed across multiple parts of the network~\cite{yao2024machine}.
Knowledge-boundary methods create negative examples to teach models to avoid certain responses \cite{li2025knowledge, liu2024dellma}. Contrastive unlearning pairs forgetting targets with similar but acceptable content to refine decision boundaries \cite{chen2023unlearn}. Dataset-filtering approaches reconstruct training data while excluding unwanted information \cite{zhao2025survey}. 
Optimization-based methods, like Influence Tuning, use importance scores to identify and modify critical parameters \cite{xu2024editing}. Direct Preference Optimization (DPO) is a prominent example that assigns higher preference to neutral or refusal outputs than to the original memorized responses~\cite{rafailov2023direct}. Negative Preference Optimization (NPO) further streamlines this process by relying only on negative forget samples~\cite{zhang2024negative}.
These approaches vary in their effectiveness, computational requirements, and ability to preserve model performance on unrelated tasks.

\noindent However, many unlearning methods remain vulnerable to \emph{quantization-induced relearning}, where low-bit quantization can effectively erase small unlearning updates and restore behavior close to the original model. Post-training quantization is widely used to reduce the computational and storage overhead of LLMs by mapping full-precision parameters to low-bit representations~\cite{gholami2022survey,lin2024awq}.
Recent quantization methods include 8-bit and 4-bit techniques, significantly enhancing inference efficiency while preserving acceptable accuracy levels \cite{dettmers2022gpt3}. Nevertheless, quantization can adversely affect machine unlearning performance, potentially exacerbating the issue of catastrophic forgetting \cite{luo2023empirical, ranjan2026razor}. 

\noindent Formally, let $\theta$ be the original parameters and $\theta'$ the post-unlearning parameters, and let $Q(\cdot)$ denote post-training rounding quantization (elementwise or per-group). For step size $\Delta_j$ on coordinate/group $j$,
\begin{equation}
Q_{\Delta_j}(w)=\Delta_j\,\mathrm{Round}\!\left(\frac{w}{\Delta_j}\right).
\end{equation}

Quantization can mask an unlearning update when the pre- and
post-unlearning parameters remain in the same quantization bin.
For coordinate (or group) $j$, this occurs exactly when
\begin{equation}
\begin{split}
Q_{\Delta_j}(\theta'_j)
=
Q_{\Delta_j}(\theta_j)
\iff \\
\operatorname{Round}\!\left(\frac{\theta'_j}{\Delta_j}\right)
=
\operatorname{Round}\!\left(\frac{\theta_j}{\Delta_j}\right).
\label{eq:quant-bin-invariance}
\end{split}
\end{equation}
If this condition holds for a large fraction of the edited
coordinates, the quantized unlearned model can become
functionally closer to the quantized original model, allowing
part of the suppressed knowledge to re-emerge
\citep{zhang2024catastrophic}.

\noindent Another prominent line of work leverages parameter-efficient modules, e.g., adapters and LoRA layers, to localize updates and preserve the bulk of pre-trained weights~\cite{liu2025lune,li2021prefix,hu2022lora}. Moreover, even after ``erasure'', residual information can be resurrected by adversarial or carefully engineered prompts, so-called spillage, undermining any privacy guarantees~\cite{ji2024aligner, zhang2024catastrophic}.  

\noindent In summary, while existing methods, such as quantization-based unlearning, knowledge editing, and parameter-efficient tuning, have provided valuable insights into model modification, they exhibit significant limitations when applied to robust unlearning scenarios.

\section{Proposed Method: FOM-UL}
\label{sec:method}

We propose FOM-UL, a targeted framework that suppresses specific knowledge by updating a small set of layers with high forget-to-retain significance. Rather than claiming exact erasure, FOM-UL aims to approximate retraining behavior on the forget set while preserving retain-set utility and improving robustness to quantization-induced recovery.

\textbf{Problem Formulation.} Let $\mathcal{M}_\theta$ denote a pre-trained transformer-based model with parameters $\theta \in \mathbb{R}^d$. Given two datasets a forget set $D_{\mathrm{forget}}$, containing the data to be erased, and a retain set $D_{\mathrm{retain}}$, comprising knowledge that must be preserved the objective is to update the model parameters to minimize verbatim memorization (VerMem) and privacy leakage (PrivLeak) on $D_{\mathrm{forget}}$, while preserving utility and knowledge retention on $D_{\mathrm{retain}}$.

\begin{figure*}[t]
  \centering
  \includegraphics[width=\linewidth]{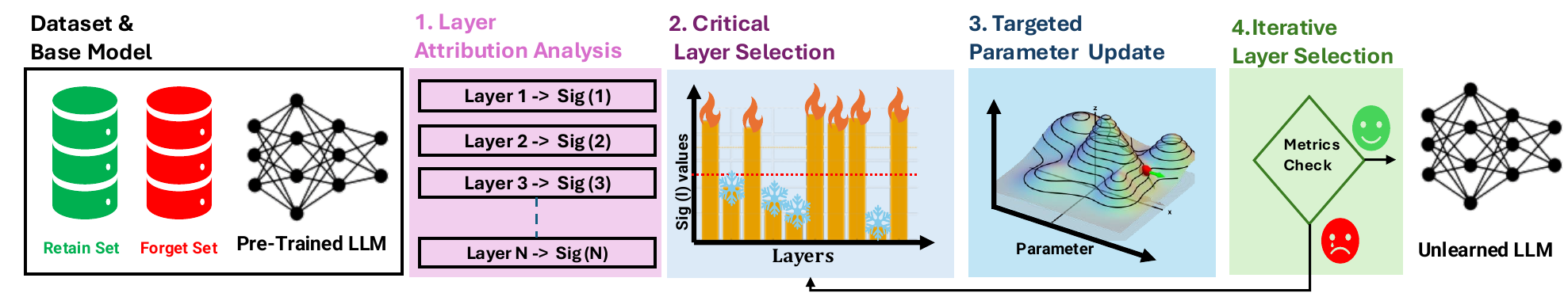}
  \caption{Overview of the Forgetting Only What Matters via Unlearning Layers (FOM-UL) framework.}
  \label{fig:method-main}
\vspace{-0.4cm}
\end{figure*}

\textbf{Motivation.}
Modern LLMs, such as Llama-2, are typically \emph{decoder-only} autoregressive transformers with $L$ stacked layers. For an input prefix $x_{1:t}$, each layer $\ell$ updates hidden states $h^{(\ell)}_{1:t}$ via (i) \emph{multi-head self-attention} (MHSA) and (ii) a position-wise \emph{feed-forward network} (FFN), coupled with residual connections and normalization. In MHSA, each \emph{attention head} performs a query-key-value interaction to form a weighted mixture of contextual token representations, and the layer aggregates diverse dependency patterns (e.g., local syntax and long-range factual cues) across heads.
Importantly, information is \emph{not uniformly distributed} across the network: lower layers tend to encode lexical/syntactic features, intermediate layers increasingly represent semantic relations, and deeper layers are more directly coupled to the final \emph{logits} (next-token scores) used for generation. Consequently, memorized or sensitive content can be disproportionately concentrated in a subset of layers/heads that exert outsized influence on specific next-token predictions. This heterogeneous localization motivates our work. By attributing a forget objective to the layers that most affect the target prediction and selectively updating only those layers, FOM-UL can induce targeted forgetting with reduced collateral utility loss, while producing parameter shifts that are more resilient to quantization-induced reversal than diffuse, global updates.

\textbf{Proposed Method Overview.} Figure~\ref{fig:method-main} presents the complete FOM-UL pipeline. First, we select the datasets, a \emph{Forget} set (e.g., copyrighted Harry Potter text) and a \emph{Retain} set (e.g., fandom wiki and general knowledge). Next, a layer‐attribution analysis identifies the transformer layers most responsible for encoding the sensitive content. A binary saliency mask is then generated (for $n$ layers) to \emph{freeze} layers with low attribution and \emph{unlock} only the high‐attribution layers for updates. Targeted unlearning alternates between (i) gradient ascent on the Forget set to maximize loss on unwanted content and (ii) gradient descent on the Retain set to preserve essential knowledge, applied solely within the selected layers. Finally, the resulting model is evaluated on verbatim memorization, privacy leakage, knowledge retention, and general utility to verify robust, quantization‐resilient forgetting.

\textbf{Identifying Key Layers.} We consider a transformer-based LLM with parameters $\theta=\{\theta^{(1)},\dots,\theta^{(L)}\}$ grouped by layer, and let $x_{1:t}$ denote an input prefix. We define $w^{*}_{t+1}=\arg\max_{w} P_{\theta}(w\mid x_{1:t})$ as the model’s top-1 next-token prediction and compute its baseline probability
\begin{equation}
\hat{p}\;=\;P_{\theta}\!\left(w^{*}_{t+1}\mid x_{1:t}\right).
\end{equation}
To localize where this prediction is formed, we perform \emph{layer-wise ablation}: for each layer $\ell\in\{1,\dots,L\}$, we intervene on the forward pass by removing (e.g., zeroing) the contribution of layer $\ell$ to the residual stream, and recompute the same next-token probability under this ablated model, denoted by $P_{\theta\setminus \ell}$:
\begin{equation}
\hat{p}_{\ell}\;=\;P_{\theta\setminus \ell}\!\left(w^{*}_{t+1}\mid x_{1:t}\right),
\qquad
\end{equation}

To prioritize layers that enable effective forgetting \emph{with minimal retain-set disruption}, FOM-UL further computes a gradient-based \emph{significance score}. Let $\mathcal{L}_{\text{forget}}(\theta;B_f)$ and $\mathcal{L}_{\text{retain}}(\theta;B_r)$ be the losses on forget and retain batches $B_f\subset\mathcal{D}_{f}$ and $B_r\subset\mathcal{D}_{r}$, respectively. We define the per-layer gradient magnitudes
\begin{equation}
I(\ell)\;=\;\left\|\nabla_{\theta^{(\ell)}}\,\mathcal{L}_{\text{forget}}(\theta;B_f)\right\|_{2},
\end{equation}
\begin{equation}
I_r(\ell)\;=\;\left\|\nabla_{\theta^{(\ell)}}\,\mathcal{L}_{\text{retain}}(\theta;B_r)\right\|_{2},
\end{equation}
and the normalized forget-to-retain trade-off score
\begin{equation}
\mathrm{Sig}(\ell)\;=\;\frac{I(\ell)}{I_r(\ell)+\varepsilon},
\end{equation}
where $\varepsilon>0$ is a small constant for numerical stability. Intuitively, high $\mathrm{Sig}(\ell)$ identifies layers that are highly responsive to the forgetting objective while being comparatively insensitive to the retain objective. In the first stage, FOM-UL selects the candidate set
\begin{equation}
S \;=\;\left\{\ell\in\{1,\dots,L\}\;:\;\mathrm{Sig}(\ell)>\tau\right\},
\end{equation}
where $\tau$ is a tunable threshold controlling the layer budget and the conservativeness of the update set (Detailed in Appendix~\ref{app:detailed-method}.

\textbf{Selective Layer Update with Forget, Mismatch, and Retain Losses.} FOM-UL performs targeted updates only on layers $\ell \in S$. The update is guided by three loss components (Detailed in Appendix~\ref{app:losses}): (i) the forgetting loss $\mathcal{L}_{\text{forget}}$ to remove memorized patterns, (ii) the mismatch loss $\mathcal{L}_{\text{mismatch}}$ to diverge from original outputs, and (iii) the retain loss $\mathcal{L}_{\text{retain}}$ to preserve general utility. For each selected layer, the parameter update is:
\begin{equation}
\label{all-losses}
\begin{split}
\theta_{t+1}^{(\ell)} = \theta_t^{(\ell)} 
+ \eta_F \nabla_{\theta^{(\ell)}} \mathcal{L}_{\text{forget}} \\
+ \eta_M \nabla_{\theta^{(\ell)}} \mathcal{L}_{\text{mismatch}} 
- \eta_R \nabla_{\theta^{(\ell)}} \mathcal{L}_{\text{retain}},
\end{split}
\end{equation}
where $\eta_F$, $\eta_M$, and $\eta_R$ are the respective learning rates. Layers not in $S$ remain frozen:
\begin{equation}
\theta_{t+1}^{(\ell)} = \theta_t^{(\ell)}, ~\forall \ell \notin S.
\end{equation}

\textbf{Iterative Expansion and Stopping Criteria.} FOM-UL proceeds iteratively: after an initial selection of top $k$ layers, ranked by $\mathrm{Sig}(\ell)$, that cross the threshold $\tau$, if the forgetting objectives (e.g., VerMem below threshold $< 0.05$) are unmet, the set $S$ is expanded by adding the next most significant layer based on an additional significance scoring, i.e.,
$S' = S \cup \{arg \max_{\ell \notin S} \mathrm{Sig}(\ell)\}.$
This continues until the forgetting metric converges or a maximum number of epochs is reached. This prevents aggressive updates early on and reduces the risk of unintended utility loss. Extended justifications and details can be found in Appendix~\ref{app:iter-clear}.

The pseudocode is provided in Appendix~\ref{app:pseudo}, additional methodological details are described in Appendix~\ref{app:detailed-method}, and formal justifications for the proposed layer selection strategy (Lemma~1) and iterative unlearning procedure (Lemma~2) are presented in Appendix~\ref{app:lemma}.

\section{Experiments}
\label{sec:experiments}

\subsection{Experimental Setup}
\vspace{-0.1cm}
Detailed implementation details, hyperparameters, and metric definitions are provided in Appendix~\ref{app:exp-setup} and Appendix~\ref{app:metrics}.

\providecommand{\pmv}[2]{#1{\scriptsize$\pm$#2}}
\begin{table*}[t!]
\centering
\scriptsize
\setlength{\tabcolsep}{2.0pt}
\renewcommand{\arraystretch}{1.12}
\caption{
\textbf{Main unlearning results on TOFU-World Facts.}
Results are reported as mean$\pm$std over three run. M1/M2 are ROUGE-based residual memorization scores, M3 measures privacy leakage distance to the retrained model, and M4 measures retain-set utility. Values are reported on a compact $0$--$10$ scale for readability; multiplying by $10$ converts them to a $0$--$100$ scale.
}
\label{tab:FOM-UL_main_results}
\resizebox{\textwidth}{!}{
\begin{tabular}{l|cccc|cccc|cccc}
\toprule
\multirow{2}{*}{\textbf{Method}}
& \multicolumn{4}{c|}{\textbf{GPT-2}}
& \multicolumn{4}{c|}{\textbf{Llama-3.2-1B}}
& \multicolumn{4}{c}{\textbf{Gemma-3-1B}} \\
\cmidrule(lr){2-5}\cmidrule(lr){6-9}\cmidrule(lr){10-13}
& M1$\downarrow$ & M2$\downarrow$ & M3$\rightarrow 0$ & M4$\uparrow$
& M1$\downarrow$ & M2$\downarrow$ & M3$\rightarrow 0$ & M4$\uparrow$
& M1$\downarrow$ & M2$\downarrow$ & M3$\rightarrow 0$ & M4$\uparrow$ \\
\midrule
Vanilla Model
& \pmv{9.00}{0.08} & \pmv{6.10}{0.06} & \pmv{9.50}{0.12} & \textbf{\pmv{2.90}{0.03}}
& \pmv{5.80}{0.07} & \pmv{5.80}{0.07} & \pmv{8.80}{0.15} & \textbf{\pmv{2.90}{0.03}}
& \pmv{5.60}{0.06} & \pmv{5.60}{0.06} & \pmv{8.50}{0.14} & \textbf{\pmv{2.90}{0.03}} \\
\midrule
GA$_{\mathrm{GDR}}$
& \pmv{4.47}{0.10} & \pmv{4.49}{0.10} & \pmv{5.04}{0.15} & \pmv{2.07}{0.04}
& \pmv{4.24}{0.09} & \pmv{4.24}{0.09} & \pmv{4.88}{0.13} & \pmv{2.07}{0.04}
& \pmv{4.16}{0.09} & \pmv{4.16}{0.09} & \pmv{4.80}{0.13} & \pmv{2.14}{0.04} \\

SURE + GA
& \pmv{2.14}{0.06} & \pmv{2.16}{0.06} & \pmv{4.44}{0.12} & \pmv{1.96}{0.04}
& \pmv{2.08}{0.05} & \pmv{2.08}{0.05} & \pmv{4.24}{0.11} & \pmv{1.98}{0.04}
& \pmv{2.06}{0.05} & \pmv{2.06}{0.05} & \pmv{4.20}{0.11} & \pmv{1.98}{0.04} \\

\rowcolor{fomulblue}
FOM-UL + GA
& \pmv{3.02}{0.07} & \pmv{3.06}{0.07} & \pmv{3.08}{0.10} & \pmv{2.34}{0.03}
& \pmv{2.96}{0.06} & \pmv{2.98}{0.06} & \pmv{2.92}{0.09} & \pmv{2.42}{0.03}
& \pmv{2.94}{0.06} & \pmv{2.96}{0.06} & \pmv{2.88}{0.09} & \pmv{2.42}{0.03} \\
\midrule
NPO$_{\mathrm{GDR}}$
& \pmv{2.14}{0.08} & \pmv{2.16}{0.08} & \pmv{8.10}{0.18} & \pmv{2.07}{0.04}
& \pmv{2.06}{0.07} & \pmv{2.06}{0.07} & \pmv{6.94}{0.18} & \pmv{2.36}{0.04}
& \pmv{2.06}{0.07} & \pmv{2.06}{0.07} & \pmv{6.60}{0.17} & \pmv{2.40}{0.04} \\

SURE + NPO
& \pmv{1.60}{0.05} & \pmv{1.62}{0.05} & \pmv{3.64}{0.11} & \pmv{1.92}{0.04}
& \pmv{1.60}{0.04} & \pmv{1.60}{0.04} & \pmv{3.40}{0.10} & \pmv{1.96}{0.04}
& \pmv{1.60}{0.04} & \pmv{1.60}{0.04} & \pmv{3.54}{0.10} & \pmv{1.96}{0.04} \\

\rowcolor{fomulblue}
FOM-UL + NPO
& \pmv{1.68}{0.05} & \pmv{1.70}{0.05} & \pmv{2.84}{0.09} & \pmv{2.28}{0.04}
& \pmv{1.58}{0.04} & \pmv{1.60}{0.04} & \pmv{2.66}{0.08} & \pmv{2.34}{0.04}
& \pmv{1.58}{0.04} & \pmv{1.60}{0.04} & \pmv{2.60}{0.08} & \pmv{2.38}{0.04} \\
\midrule
KLD
& \pmv{4.47}{0.11} & \pmv{4.52}{0.11} & \pmv{7.92}{0.18} & \pmv{1.92}{0.05}
& \pmv{4.04}{0.10} & \pmv{4.04}{0.10} & \pmv{6.00}{0.16} & \pmv{1.98}{0.04}
& \pmv{4.00}{0.10} & \pmv{4.00}{0.10} & \pmv{5.60}{0.16} & \pmv{1.98}{0.04} \\

SURE + KLD
& \pmv{1.96}{0.05} & \pmv{1.98}{0.05} & \pmv{4.00}{0.12} & \pmv{2.16}{0.04}
& \pmv{1.90}{0.04} & \pmv{1.90}{0.04} & \pmv{3.84}{0.11} & \pmv{2.12}{0.04}
& \pmv{1.88}{0.04} & \pmv{1.88}{0.04} & \pmv{3.80}{0.11} & \pmv{2.12}{0.04} \\

\rowcolor{fomulblue}
FOM-UL + KLD
& \pmv{2.74}{0.07} & \pmv{2.78}{0.07} & \pmv{3.18}{0.10} & \pmv{2.30}{0.04}
& \pmv{2.62}{0.06} & \pmv{2.64}{0.06} & \pmv{3.04}{0.09} & \pmv{2.36}{0.04}
& \pmv{2.60}{0.06} & \pmv{2.62}{0.06} & \pmv{3.00}{0.09} & \pmv{2.36}{0.04} \\
\midrule
ReLearn
& \pmv{5.24}{0.12} & \pmv{5.36}{0.12} & \pmv{5.00}{0.14} & \pmv{2.00}{0.04}
& \pmv{5.04}{0.11} & \pmv{5.04}{0.11} & \pmv{4.80}{0.13} & \pmv{2.14}{0.04}
& \pmv{5.02}{0.11} & \pmv{5.02}{0.11} & \pmv{4.80}{0.13} & \pmv{2.14}{0.04} \\

LUNAR
& \textbf{\pmv{1.54}{0.04}} & \pmv{1.60}{0.04} & \pmv{4.24}{0.12} & \pmv{1.86}{0.05}
& \pmv{1.28}{0.03} & \pmv{1.28}{0.03} & \pmv{4.10}{0.11} & \pmv{2.06}{0.04}
& \pmv{1.22}{0.03} & \textbf{\pmv{1.22}{0.03}} & \pmv{4.10}{0.11} & \pmv{2.06}{0.04} \\

\rowcolor{fomuldeep}
\textbf{FOM-UL-Full (Ours)}
& \pmv{1.56}{0.03} & \textbf{\pmv{1.58}{0.03}} & \textbf{\pmv{2.42}{0.07}} & \pmv{2.84}{0.03}
& \textbf{\pmv{1.24}{0.03}} & \textbf{\pmv{1.24}{0.03}} & \textbf{\pmv{1.96}{0.06}} & \textbf{\pmv{2.90}{0.03}}
& \textbf{\pmv{1.20}{0.03}} & \pmv{1.24}{0.03} & \textbf{\pmv{1.92}{0.06}} & \textbf{\pmv{2.90}{0.03}} \\
\bottomrule
\end{tabular}
}
\vspace{-0.2cm}
\end{table*}

\textbf{Baselines.}
We compare FOM-UL against the vanilla model and a broad set of LLM unlearning baselines. Following the SURE protocol~\cite{zhang2024catastrophic}, we evaluate Gradient Ascent (GA) and Negative Preference Optimization (NPO), combined with common retain-preserving objectives such as gradient descent on the retain set (GDR) and KL regularization (KLR). GA directly reduces confidence on forget samples, while NPO treats forget examples as negative preferences. We also include recent state-of-the-art methods, including ReLearn~\cite{xu2025relearn}, which performs unlearning through data augmentation and fine-tuning, and LUNAR~\cite{shen2025llm}, which redirects internal activations. In addition, we consider parameter-efficient and quantization-aware unlearning baselines to evaluate robustness and efficiency.

\textbf{Datasets.}
We evaluate on three standard LLM unlearning benchmarks. TOFU~\cite{maini2024tofu} tests factual QA forgetting over synthetic world facts. KnowUnDo~\cite{tian2024forget} evaluates privacy- and copyright-oriented unlearning while checking whether useful knowledge is unintentionally removed. We also use MUSE, including BOOKS and NEWS. BOOKS uses the Harry Potter corpus as the forget set and FanWiki as the retain set, while NEWS contains BBC articles split into forget, retain, and holdout subsets for evaluating memorization, utility, and privacy leakage.

\textbf{Metrics.}
We follow the standard four-metric protocol used in prior work~\cite{zhang2024catastrophic}. M1 Verbatim Memorization and M2 Knowledge Memorization measure residual content from the forget set; lower values indicate stronger forgetting. M3 Privacy Leakage measures membership-inference risk and is best when close to zero. M4 Utility Preservation measures retained knowledge on the retain set, where higher values indicate better utility. Together, these metrics capture the main trade-off between erasing unwanted knowledge and preserving useful behavior.

\textbf{Models.}
We evaluate FOM-UL on multiple transformer-based LLMs, including Llama-2 7B, Llama-3.2 1B~\cite{grattafiori2024llama}, GPT-2~\cite{hanna2023does}, and Gemma-3 1B~\cite{team2024gemma}. These models cover different scales and architectures, allowing us to test whether layer-selective unlearning remains effective across model families.

\subsection{Unlearning Results}
\label{sec:Results}

\noindent \textbf{Performance Comparison.}
Table~\ref{tab:FOM-UL_main_results} shows two key findings. First, FOM-UL consistently achieves the best overall forgetting-utility balance across GPT-2, Llama-3.2-1B, and Gemma-3-1B, reducing residual memorization and privacy leakage while keeping retain-set utility close to the vanilla model. Second, the gains are stable across different base objectives, showing that the proposed layer-selection strategy improves standard GA, NPO, and KLD-style unlearning rather than depending on a single loss. Based on these results, Table~\ref{tab:table-2} evaluates the best-performing combinations across NEWS, KnowUnDo, and BOOKS.

\noindent \textbf{Evaluation of the Best-Performing combinations.}
Table~\ref{tab:table-2} shows that FOM-UL provides the strongest overall forgetting-utility trade-off across NEWS, KnowUnDo, and BOOKS: it matches or ties the best memorization scores while substantially reducing privacy leakage. It also preserves retain-set utility close to the vanilla model, indicating that the gains are not due to destructive over-unlearning.

\providecommand{\pmv}[2]{#1{\scriptsize$\pm$#2}}

\begin{table*}[t]
\centering
\scriptsize
\setlength{\tabcolsep}{2.4pt}
\renewcommand{\arraystretch}{1.12}
\caption{
\textbf{Best-performing unlearning combinations on Llama-3.2-1B across three datasets.}
Results are reported as mean$\pm$std over three runs. 
}
\vspace{-0.1cm}
\label{tab:table-2}
\resizebox{\textwidth}{!}{
\begin{tabular}{l|cccc|cccc|cccc}
\toprule
\multirow{2}{*}{\textbf{Method}}
& \multicolumn{4}{c|}{\textbf{NEWS}}
& \multicolumn{4}{c|}{\textbf{KnowUnDo}}
& \multicolumn{4}{c}{\textbf{BOOKS}} \\
\cmidrule(lr){2-5}\cmidrule(lr){6-9}\cmidrule(lr){10-13}
& M1$\downarrow$ & M2$\downarrow$ & M3$\rightarrow 0$ & M4$\uparrow$
& M1$\downarrow$ & M2$\downarrow$ & M3$\rightarrow 0$ & M4$\uparrow$
& M1$\downarrow$ & M2$\downarrow$ & M3$\rightarrow 0$ & M4$\uparrow$ \\
\midrule
Vanilla Model
& \pmv{5.80}{0.07} & \pmv{5.80}{0.07} & \pmv{8.80}{0.15} & \pmv{2.90}{0.03}
& \pmv{5.80}{0.07} & \pmv{5.80}{0.07} & \pmv{8.80}{0.15} & \pmv{2.90}{0.03}
& \pmv{5.80}{0.07} & \pmv{5.80}{0.07} & \pmv{8.80}{0.15} & \pmv{2.90}{0.03} \\
\midrule
GA$_{\mathrm{GDR}}$
& \pmv{4.24}{0.09} & \pmv{4.24}{0.09} & \pmv{4.84}{0.13} & \pmv{2.10}{0.04}
& \pmv{4.26}{0.09} & \pmv{4.26}{0.09} & \pmv{4.82}{0.13} & \pmv{2.10}{0.04}
& \pmv{4.24}{0.09} & \pmv{4.24}{0.09} & \pmv{4.84}{0.13} & \pmv{2.10}{0.04} \\

NPO$_{\mathrm{GDR}}$
& \pmv{2.16}{0.07} & \pmv{2.16}{0.07} & \pmv{8.00}{0.18} & \pmv{2.10}{0.04}
& \pmv{2.12}{0.07} & \pmv{2.12}{0.07} & \pmv{7.40}{0.17} & \pmv{2.10}{0.04}
& \pmv{2.16}{0.07} & \pmv{2.16}{0.07} & \pmv{8.00}{0.18} & \pmv{2.10}{0.04} \\

KLD
& \pmv{4.52}{0.10} & \pmv{4.52}{0.10} & \pmv{7.24}{0.17} & \pmv{1.98}{0.04}
& \pmv{4.26}{0.10} & \pmv{4.26}{0.10} & \pmv{6.86}{0.16} & \pmv{2.00}{0.04}
& \pmv{4.52}{0.10} & \pmv{4.52}{0.10} & \pmv{7.24}{0.17} & \pmv{1.98}{0.04} \\

SURE + NPO
& \pmv{1.62}{0.02} & \pmv{1.62}{0.04} & \pmv{3.68}{0.10} & \pmv{1.98}{0.04}
& \pmv{1.64}{0.03} & \pmv{1.64}{0.04} & \pmv{3.72}{0.10} & \pmv{2.00}{0.04}
& \pmv{1.62}{0.05} & \pmv{1.62}{0.04} & \pmv{3.68}{0.10} & \pmv{1.98}{0.04} \\

ReLearn
& \pmv{5.22}{0.11} & \pmv{5.22}{0.11} & \pmv{4.96}{0.13} & \pmv{2.08}{0.04}
& \pmv{5.26}{0.11} & \pmv{5.26}{0.11} & \pmv{4.96}{0.13} & \pmv{2.10}{0.04}
& \pmv{5.22}{0.11} & \pmv{5.22}{0.11} & \pmv{4.96}{0.13} & \pmv{2.08}{0.04} \\

LUNAR
& \textbf{\pmv{1.56}{0.03}} & \textbf{\pmv{1.56}{0.03}} & \pmv{4.20}{0.11} & \pmv{2.14}{0.04}
& \pmv{1.60}{0.03} & \pmv{1.60}{0.03} & \pmv{4.20}{0.11} & \pmv{2.16}{0.04}
& \textbf{\pmv{1.56}{0.03}} & \textbf{\pmv{1.56}{0.03}} & \pmv{4.20}{0.11} & \pmv{2.14}{0.04} \\

\rowcolor{fomuldeep}
\textbf{FOM-UL}
& \textbf{\pmv{1.56}{0.03}} & \textbf{\pmv{1.56}{0.02}} & \textbf{\pmv{2.64}{0.07}} & \textbf{\pmv{2.90}{0.02}}
& \textbf{\pmv{1.58}{0.03}} & \textbf{\pmv{1.58}{0.03}} & \textbf{\pmv{2.64}{0.08}} & \textbf{\pmv{2.90}{0.03}}
& \textbf{\pmv{1.56}{0.04}} & \textbf{\pmv{1.56}{0.03}} & \textbf{\pmv{2.64}{0.08}} & \textbf{\pmv{2.90}{0.03}} \\
\bottomrule
\end{tabular}
}
\vspace{-0.2cm}
\end{table*}

\begin{figure}[t]
\vspace{-0.2cm}
  \centering
  \includegraphics[width=\linewidth]{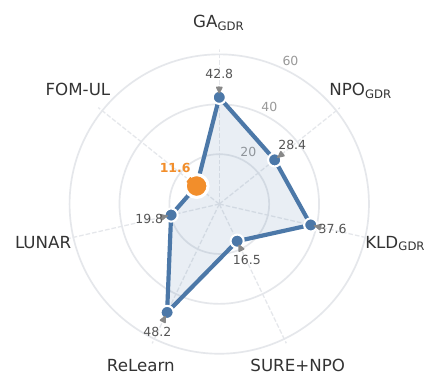}
  \caption{Attack Leakage Rate (ALR) under adversarial/jailbreak prompts. Lower ALR indicates fewer successful recoveries of forgotten on Llama-3 and TOFU dataset under adversarial prompting; FOM-UL achieves the lowest leakage score among all methods (Table~\ref{tab:fomul_jailbreak_robustness}).}
  \label{fig:robust-adv}
\vspace{-0.2cm}
\end{figure}

\begin{figure}[b!]
\vspace{-0.2cm}
  \centering
  \includegraphics[width=0.9\linewidth]{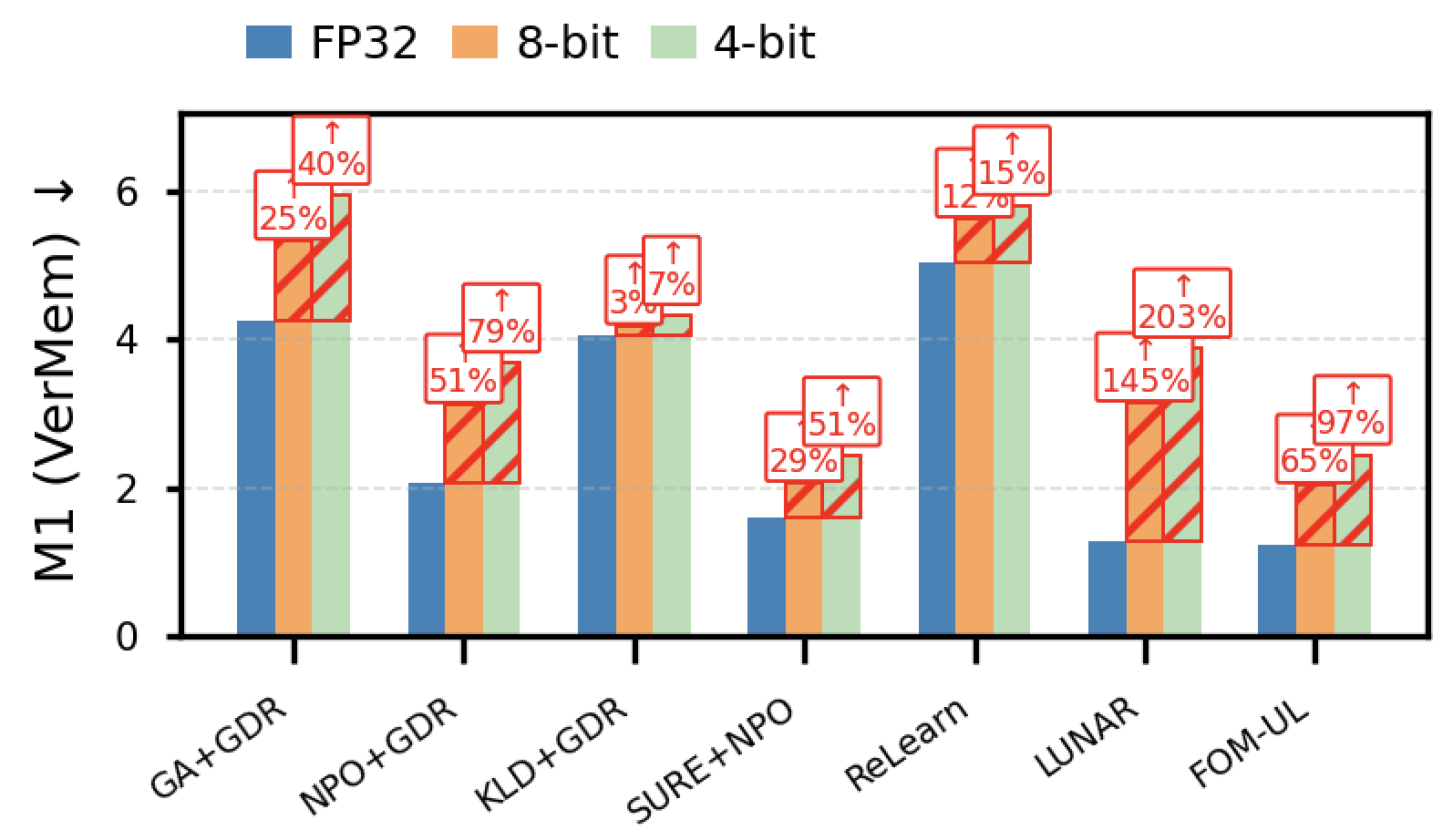}
  \caption{Quantization robustness of unlearning methods measured by M1 (VerMem, lower is better) under Full precision 32-bit (FP-32), 8-bit, and 4-bit post-training quantization; red annotations denote the relative M1 increase when moving from FP-32 to 8-bit and 4-bit.}
  \label{fig:robust-1}
\vspace{-0.2cm}
\end{figure}

\subsection{Robustness Analysis.}
\label{sec:robust}
\vspace{-0.2cm}

\begin{table}[t]
\centering
\scriptsize
\setlength{\tabcolsep}{3.0pt}
\renewcommand{\arraystretch}{1.10}
\caption{
Quantization robustness on Llama-3.2-1B using TOFU-World Facts.
Results compare 8-bit and 4-bit post-training quantization.
}
\label{tab:tofu_quant_llama3}
\resizebox{\columnwidth}{!}{
\begin{tabular}{llcccc}
\toprule
\textbf{Method.} &
\textbf{Quant.} &
\textbf{M1}$\downarrow$ &
\textbf{M2}$\downarrow$ &
\textbf{M3}$\rightarrow 0$ &
\textbf{M4}$\uparrow$ \\
\midrule
\rowcolor{quantgray}
\multicolumn{6}{c}{\textbf{8-bit Quantization}} \\
GA$_{\mathrm{GDR}}$      & 8-bit & 5.32 & 5.32 & 4.92  & 1.94 \\
NPO$_{\mathrm{GDR}}$     & 8-bit & 3.12 & 3.12 & 6.98  & 2.04 \\
KLD$_{\mathrm{GDR}}$     & 8-bit & 4.18 & 4.18 & 6.18  & 1.82 \\
SURE + NPO               & 8-bit & 2.06 & 2.06 & 3.62  & 1.84 \\
ReLearn                  & 8-bit & 5.62 & 5.62 & 4.84  & 1.92 \\
LUNAR                    & 8-bit & 3.14 & 3.14 & 4.96  & 1.72 \\
\rowcolor{fomulblue}
\textbf{FOM-UL}          & 8-bit & \textbf{2.04} & \textbf{2.04} & \textbf{$-$1.47} & \textbf{2.43} \\
\midrule
\rowcolor{quantgray}
\multicolumn{6}{c}{\textbf{4-bit Quantization}} \\
GA$_{\mathrm{GDR}}$      & 4-bit & 5.94 & 5.94 & 5.08  & 1.67 \\
NPO$_{\mathrm{GDR}}$     & 4-bit & 3.68 & 3.68 & 7.02  & 2.00 \\
KLD$_{\mathrm{GDR}}$     & 4-bit & 4.32 & 4.32 & 6.22  & 1.80 \\
SURE + NPO               & 4-bit & \textbf{2.42} & \textbf{2.42} & 3.68  & 1.80 \\
ReLearn                  & 4-bit & 5.80 & 5.80 & 4.92  & 1.87 \\
LUNAR                    & 4-bit & 3.88 & 3.88 & 5.16  & 1.70 \\
\rowcolor{fomuldeep}
\textbf{FOM-UL}          & 4-bit & 2.44 & 2.44 & \textbf{$-$2.00} & \textbf{2.04} \\
\bottomrule
\end{tabular}
}
\vspace{-0.4cm}
\end{table}

\noindent \textbf{Adversarial/Jailbreak Robustness.}
As shown in Figure~\ref{fig:robust-adv}, FOM-UL achieves the lowest ALR of $11.6\%$, compared to $16.5\%$ for SURE+NPO and $19.8\%$ for LUNAR. This indicates stronger resistance to jailbreak-based recovery of forgotten knowledge. These results show that FOM-UL remains effective not only under clean prompts, but also under adversarial extraction attempts. The complete results are provided in Appendix~\ref{app:robust}.

\noindent \textbf{Quantization Robustness.}
Table~\ref{tab:tofu_quant_llama3} shows that 4-bit quantization generally weakens unlearning by increasing residual memorization relative to 8-bit precision. FOM-UL reduces residual memorization and improves privacy-parity metrics relative to retraining. However, its negative M3 values indicate deviation from the retrained privacy baseline rather than improved privacy. We therefore interpret M3 by distance to zero and report these values as evidence of privacy-behavior shift under aggressive quantization, while the main robustness gain of FOM-UL is strongest on memorization and utility. Additional analysis is provided in Appendices~\ref{app:neg-m3} and~\ref{appendix:FOM-UL_better}.

\subsection{Runtime Performance}
\label{sec:runtime}
\vspace{-0.1cm}

\begin{table}[t]
\centering
\scriptsize
\setlength{\tabcolsep}{2.7pt}
\renewcommand{\arraystretch}{1.08}
\caption{
\textbf{Efficiency comparison on TOFU-World Facts with Llama-2.}
FOM-UL achieves a practical balance between trainable parameter size, GPU memory, and runtime while avoiding full-model updates.
}
\label{tab:efficiency_comparison}
\resizebox{\columnwidth}{!}{
\begin{tabular}{lccc}
\toprule
\textbf{Method} & \textbf{Trainable Params.} & \textbf{GPU Mem.} & \textbf{Time} \\
\midrule
GA$_{\mathrm{GDR}}$      & 7B      & 16 GB        & 4 hrs \\
NPO$_{\mathrm{GDR}}$     & 7B      & 8 GB         & 3.2 hrs \\
KLD                      & 7B      & 32 GB        & 4 hrs \\
\rowcolor{softgray}
SURE + NPO               & \textbf{1.7M} & 8 GB   & \textbf{15 min} \\
ReLearn                  & 2B      & 16 GB        & 35 min \\
LUNAR                    & 1.75M   & 8 GB   & 20 min \\
\rowcolor{fomuldeep}
\textbf{FOM-UL}          & 7M      & \textbf{8 GB} & 20 min \\
\bottomrule
\end{tabular}
}
\vspace{-0.15cm}
\end{table}

Table~\ref{tab:efficiency_comparison} highlights the computational efficiency of different unlearning methods. Full-model approaches such as GA and KLD incur high memory and runtime costs due to updating all parameters, making them less practical for large-scale deployment. In contrast, parameter-efficient methods significantly reduce resource requirements. Notably, FOM-UL achieves competitive unlearning performance while updating only a small subset of parameters, requiring the lowest GPU memory (6--8\,GB) and short runtimes (10--30 minutes). Compared to other efficient baselines such as SURE+NPO, ReLearn, and LUNAR, FOM-UL offers a more favorable balance between parameter count, memory usage, and execution time, demonstrating its practicality for scalable and resource-constrained unlearning settings. Details regarding parameter count in Appendix~\ref{app:run-para}, and reproducibility and hyperparameter configurations are provided in Appendix~\ref{app:reproduce}.

\begin{figure*}[t]
\vspace{-0.2cm}
  \centering
  \includegraphics[width=\linewidth]{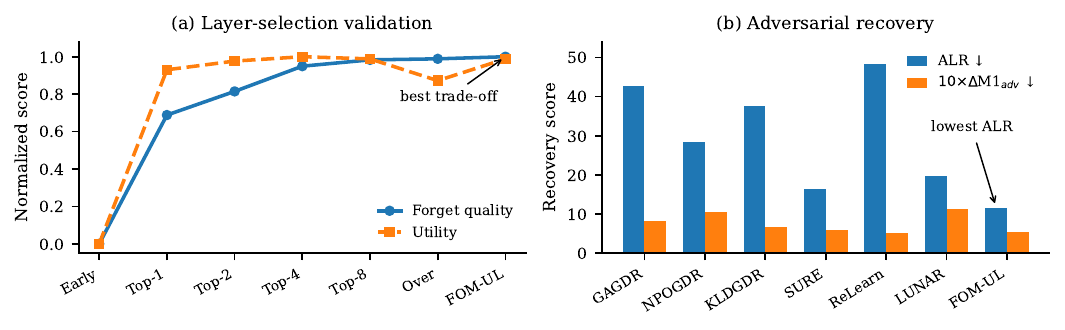}
  \vspace{-0.1cm}
  \caption{\textbf{Layer-selection and robustness validation.} FOM-UL achieves the best forgetting-utility trade-off across layer choices and the lowest adversarial recovery among evaluated baselines.}
  \label{fig:training-steps}
\vspace{-0.2cm}
\end{figure*}

\subsection{Ablation Study}
\label{sec:ablation}

\begin{table}[t]
\centering
\scriptsize
\setlength{\tabcolsep}{3.2pt}
\renewcommand{\arraystretch}{1.10}
\caption{ Ablation study of FOM-UL components on TOFU-World Facts with Llama-3.2-1B.}
\label{tab:FOM-UL_ablation}
\resizebox{\columnwidth}{!}{
\begin{tabular}{lcccc}
\toprule
\textbf{Variant} &
\textbf{M1}$\downarrow$ &
\textbf{M2}$\downarrow$ &
\textbf{M3}$\rightarrow 0$ &
\textbf{M4}$\uparrow$ \\
\midrule
Forget only
& 1.56 & 1.56 & 8.31  & 2.83 \\
Retain only
& 2.28 & 2.28 & 7.30  & \textbf{4.62} \\
Mismatch only
& 1.56 & 1.56 & $-$6.44 & 4.66 \\
\rowcolor{fomuldeep}
\textbf{FOM-UL full}
& \textbf{1.24} & \textbf{1.24} & \textbf{2.04} & 2.90 \\
\bottomrule
\end{tabular}
}
\vspace{-0.2cm}
\end{table}

The ablation results in table-\ref{tab:FOM-UL_ablation} show that using only a single component of FOM-UL leads to suboptimal trade-offs between forgetting and utility. While \texttt{FOM-UL\_forget\_only} achieves low memorization, it suffers from high privacy leakage, and \texttt{FOM-UL\_retain\_only} preserves utility at the cost of weaker forgetting. \texttt{FOM-UL\_mismatch\_only} improves utility but introduces instability in privacy behavior. In contrast, \textbf{FOM-UL\_full} consistently balances effective forgetting with stable privacy and utility, confirming the necessity of jointly optimizing all FOM-UL components. A comprehensive sensitivity analysis is provided in Appendix~\ref{app:sensitivity-analysis}.

\section{Discussion}
\label{sec:disscussion}
\vspace{-0.1cm}

Figure~\ref{fig:training-steps} results show that high-$\mathrm{Sig}(\ell)$ layer selection is more effective than early-only or over-expanded updates, and that FOM-UL also yields the lowest attack leakage rate under jailbreak-style recovery prompts.
Full normalization and aggregation details are provided in Appendix~\ref{app:norm-clear}.

Overall, \textbf{FOM-UL} consistently reduces residual targeted knowledge with limited utility loss, and remains robust under aggressive post-training quantization, where many conventional unlearning methods exhibit severe reversals due to discretizations artifacts. 

\textbf{Extended Evaluation on Diverse LLM Architectures.} We further validate generality across LLaMA-2, LLaMA-3, GPT-2, and Gemma-3 on the Tofu World Facts benchmark. Across architectures, FOM-UL reduces residual memorization and privacy leakage while maintaining strong retained utility. Notably, under 4-bit quantization on Llama-3, FOM-UL achieves a residual memorization rate of $1.22\%$, highlighting resilience to quantization-induced recovery that commonly affects global unlearning methods~\cite{zhang2024catastrophic}.

FOM-UL’s efficiency stems from: (i) \textbf{selective updates} restricted to a small set of responsible layers, (ii) \textbf{faster optimization} due to fewer trainable parameters per step, and (iii) \textbf{lower memory footprint}, enabling larger batches and reduced check-pointing. 
These gains make FOM-UL practical at scale and competitive with quantization-centric approaches in both completeness and utility~\cite{zhang2024catastrophic}, aligning with real-world compliance needs under evolving privacy regulations.

Additional implementation details and variability analyses are provided in the Appendix~\ref{app:detailed-method}. We also include ablations probing quantization-induced recovery by varying the update scope (top $k$ layers through full-model), learning rates, and regularization strength. In contrast to diffuse global updates, FOM-UL concentrates stronger edits on attribution-identified layers, yielding robust forgetting with limited collateral damage~\cite{zhang2024catastrophic}. 

\section{Conclusion}
\label{sec:conclusion}
\vspace{-0.1cm}
We introduced \textbf{FOM-UL}, a layer-selective framework for approximate LLM unlearning. FOM-UL identifies layers with high forget-set influence and low retain-set sensitivity, then restricts updates to this small subset rather than modifying the full model. This design improves the forgetting-utility trade-off while reducing unnecessary parameter changes and computational cost.

\noindent Across TOFU, KnowUnDo, and MUSE-style evaluations, FOM-UL reduces residual memorization compared with strong unlearning baselines while preserving retain-set utility. Its concentrated updates also improve robustness under adversarial recovery prompts and low-bit post-training quantization, where diffuse unlearning updates can be weakened or erased. These results support layer-level forget-retain localization as a practical direction for efficient and deployment-aware unlearning, while leaving formal guarantees of complete erasure to future work.

\section*{Limitations}
While FOM-UL improves targeted forgetting and quantization robustness by concentrating updates on a small subset of influential layers, its effectiveness depends on the reliability of layer-attribution signals, which can vary across prompts, domains, and evaluation setups. Moreover, FOM-UL is not a formal guarantee of erasure: highly entangled or redundantly encoded knowledge may require expanding the updated layer set or additional iterations, which can increase compute and introduce utility trade-offs under more adversarial or distribution-shifted settings. We discuss \textbf{ethical considerations}, safeguards, intended use, and risk mitigation in Appendix~\ref{app:ethics}.

\bibliography{custom}

@article{zhang2024catastrophic,
  title={Catastrophic failure of llm unlearning via quantization},
  author={Zhang, Zhiwei and Wang, Fali and Li, Xiaomin and Wu, Zongyu and Tang, Xianfeng and Liu, Hui and He, Qi and Yin, Wenpeng and Wang, Suhang},
  journal={arXiv preprint arXiv:2410.16454},
  year={2024}
}

@article{zhao2023survey,
  title={A survey of large language models},
  author={Zhao, Wayne Xin and Zhou, Kun and Li, Junyi and Tang, Tianyi and Wang, Xiaolei and Hou, Yupeng and Min, Yingqian and Zhang, Beichen and Zhang, Junjie and Dong, Zican and others},
  journal={arXiv preprint arXiv:2303.18223},
  volume={1},
  number={2},
  year={2023}
}

@article{geng2025comprehensive,
title={A comprehensive survey of machine unlearning techniques for large language models},
  author={Geng, Jiahui and Li, Qing and Woisetschlaeger, Herbert and Chen, Zongxiong and Cai, Fengyu and Wang, Yuxia and Nakov, Preslav and Jacobsen, Hans-Arno and Karray, Fakhri},
  journal={arXiv preprint arXiv:2503.01854},
  year={2025}
}

@article{yao2024machine,
  title={Machine unlearning of pre-trained large language models},
  author={Yao, Jin and Chien, Eli and Du, Minxin and Niu, Xinyao and Wang, Tianhao and Cheng, Zezhou and Yue, Xiang},
  journal={arXiv preprint arXiv:2402.15159},
  year={2024}
}

@article{council2022general,
  title={General data protection regulation},
  author={Council, E and others},
  journal={Official Journal of the European},
  year={2022}
}

@article{jang2022knowledge,
  title={Knowledge unlearning for mitigating privacy risks in language models},
  author={Jang, Joel and Yoon, Dongkeun and Yang, Sohee and Cha, Sungmin and Lee, Moontae and Logeswaran, Lajanugen and Seo, Minjoon},
  journal={arXiv preprint arXiv:2210.01504},
  year={2022}
}

@article{huang2024offset,
  title={Offset unlearning for large language models},
  author={Huang, James Y and Zhou, Wenxuan and Wang, Fei and Morstatter, Fred and Zhang, Sheng and Poon, Hoifung and Chen, Muhao},
  journal={arXiv preprint arXiv:2404.11045},
  year={2024}
}

@article{liu2024large,
  title={Large language model unlearning via embedding-corrupted prompts},
  author={Liu, Chris and Wang, Yaxuan and Flanigan, Jeffrey and Liu, Yang},
  journal={Advances in Neural Information Processing Systems},
  volume={37},
  pages={118198--118266},
  year={2024}
}

@article{rafailov2023direct,
  title={Direct preference optimization: Your language model is secretly a reward model},
  author={Rafailov, Rafael and Sharma, Archit and Mitchell, Eric and Manning, Christopher D and Ermon, Stefano and Finn, Chelsea},
  journal={Advances in Neural Information Processing Systems},
  volume={36},
  pages={53728--53741},
  year={2023}
}

@article{maini2024tofu,
  title={Tofu: A task of fictitious unlearning for llms},
  author={Maini, Pratyush and Feng, Zhili and Schwarzschild, Avi and Lipton, Zachary C and Kolter, J Zico},
  journal={arXiv preprint arXiv:2401.06121},
  year={2024}
}

@article{zhang2024negative,
  title={Negative preference optimization: From catastrophic collapse to effective unlearning},
  author={Zhang, Ruiqi and Lin, Licong and Bai, Yu and Mei, Song},
  journal={arXiv preprint arXiv:2404.05868},
  year={2024}
}

@article{lu2022quark,
  title={Quark: Controllable text generation with reinforced unlearning},
  author={Lu, Ximing and Welleck, Sean and Hessel, Jack and Jiang, Liwei and Qin, Lianhui and West, Peter and Ammanabrolu, Prithviraj and Choi, Yejin},
  journal={Advances in neural information processing systems},
  volume={35},
  pages={27591--27609},
  year={2022}
}

@inproceedings{bourtoule2021machine,
  title={Machine unlearning},
  author={Bourtoule, Lucas and Chandrasekaran, Varun and Choquette-Choo, Christopher A and Jia, Hengrui and Travers, Adelin and Zhang, Baiwu and Lie, David and Papernot, Nicolas},
  booktitle={2021 IEEE symposium on security and privacy (SP)},
  pages={141--159},
  year={2021},
  organization={IEEE}
}

@article{yao2024large,
  title={Large language model unlearning},
  author={Yao, Yuanshun and Xu, Xiaojun and Liu, Yang},
  journal={Advances in Neural Information Processing Systems},
  volume={37},
  pages={105425--105475},
  year={2024}
}

@inproceedings{golatkar2020forgetting,
  title={Forgetting outside the box: Scrubbing deep networks of information accessible from input-output observations},
  author={Golatkar, Aditya and Achille, Alessandro and Soatto, Stefano},
  booktitle={Computer Vision--ECCV 2020: 16th European Conference, Glasgow, UK, August 23--28, 2020, Proceedings, Part XXIX 16},
  pages={383--398},
  year={2020},
  organization={Springer}
}

@incollection{gholami2022survey,
  title={A survey of quantization methods for efficient neural network inference},
  author={Gholami, Amir and Kim, Sehoon and Dong, Zhen and Yao, Zhewei and Mahoney, Michael W and Keutzer, Kurt},
  booktitle={Low-power computer vision},
  pages={291--326},
  year={2022},
  publisher={Chapman and Hall/CRC}
}

@article{hu2022lora,
  title={Lora: Low-rank adaptation of large language models.},
  author={Hu, Edward J and Shen, Yelong and Wallis, Phillip and Allen-Zhu, Zeyuan and Li, Yuanzhi and Wang, Shean and Wang, Lu and Chen, Weizhu and others},
  journal={ICLR},
  volume={1},
  number={2},
  pages={3},
  year={2022}
}

@article{li2021prefix,
  title={Prefix-tuning: Optimizing continuous prompts for generation},
  author={Li, Xiang Lisa and Liang, Percy},
  journal={arXiv preprint arXiv:2101.00190},
  year={2021}
}

@article{dettmers2022gpt3,
  title={Gpt3. int8 (): 8-bit matrix multiplication for transformers at scale},
  author={Dettmers, Tim and Lewis, Mike and Belkada, Younes and Zettlemoyer, Luke},
  journal={Advances in neural information processing systems},
  volume={35},
  pages={30318--30332},
  year={2022}
}

@article{luo2023empirical,
  title={An empirical study of catastrophic forgetting in large language models during continual fine-tuning},
  author={Luo, Yun and Yang, Zhen and Meng, Fandong and Li, Yafu and Zhou, Jie and Zhang, Yue},
  journal={arXiv preprint arXiv:2308.08747},
  year={2023}
}

@article{ji2024aligner,
  title={Aligner: Efficient alignment by learning to correct},
  author={Ji, Jiaming and Chen, Boyuan and Lou, Hantao and Hong, Donghai and Zhang, Borong and Pan, Xuehai and Qiu, Tianyi Alex and Dai, Juntao and Yang, Yaodong},
  journal={Advances in Neural Information Processing Systems},
  volume={37},
  pages={90853--90890},
  year={2024}
}

@article{hanna2023does,
  title={How does GPT-2 compute greater-than?: Interpreting mathematical abilities in a pre-trained language model},
  author={Hanna, Michael and Liu, Ollie and Variengien, Alexandre},
  journal={Advances in Neural Information Processing Systems},
  volume={36},
  pages={76033--76060},
  year={2023}
}

@article{grattafiori2024llama,
  title={The llama 3 herd of models},
  author={Grattafiori, Aaron and Dubey, Abhimanyu and Jauhri, Abhinav and Pandey, Abhinav and Kadian, Abhishek and Al-Dahle, Ahmad and Letman, Aiesha and Mathur, Akhil and Schelten, Alan and Vaughan, Alex and others},
  journal={arXiv preprint arXiv:2407.21783},
  year={2024}
}

@article{team2024gemma,
  title={Gemma: Open models based on gemini research and technology},
  author={Team, Gemma and Mesnard, Thomas and Hardin, Cassidy and Dadashi, Robert and Bhupatiraju, Surya and Pathak, Shreya and Sifre, Laurent and Rivi{\`e}re, Morgane and Kale, Mihir Sanjay and Love, Juliette and others},
  journal={arXiv preprint arXiv:2403.08295},
  year={2024}
}

@article{ilharco2022editing,
  title={Editing models with task arithmetic},
  author={Ilharco, Gabriel and Ribeiro, Marco Tulio and Wortsman, Mitchell and Gururangan, Suchin and Schmidt, Ludwig and Hajishirzi, Hannaneh and Farhadi, Ali},
  journal={arXiv preprint arXiv:2212.04089},
  year={2022}
}

@article{liu2024dellma,
  title={Dellma: Decision making under uncertainty with large language models},
  author={Liu, Ollie and Fu, Deqing and Yogatama, Dani and Neiswanger, Willie},
  journal={arXiv preprint arXiv:2402.02392},
  year={2024}
}

@article{chen2023unlearn,
  title={Unlearn what you want to forget: Efficient unlearning for llms},
  author={Chen, Jiaao and Yang, Diyi},
  journal={arXiv preprint arXiv:2310.20150},
  year={2023}
}

@article{zhang2024does,
  title={Does your llm truly unlearn? an embarrassingly simple approach to recover unlearned knowledge},
  author={Zhang, Zhiwei and Wang, Fali and Li, Xiaomin and Wu, Zongyu and Tang, Xianfeng and Liu, Hui and He, Qi and Yin, Wenpeng and Wang, Suhang},
  journal={arXiv e-prints},
  pages={arXiv--2410},
  year={2024}
}

@article{lin2024awq,
  title={Awq: Activation-aware weight quantization for on-device llm compression and acceleration},
  author={Lin, Ji and Tang, Jiaming and Tang, Haotian and Yang, Shang and Chen, Wei-Ming and Wang, Wei-Chen and Xiao, Guangxuan and Dang, Xingyu and Gan, Chuang and Han, Song},
  journal={Proceedings of Machine Learning and Systems},
  volume={6},
  pages={87--100},
  year={2024}
}

@article{liu2025rethinking,
  title={Rethinking machine unlearning for large language models},
  author={Liu, Sijia and Yao, Yuanshun and Jia, Jinghan and Casper, Stephen and Baracaldo, Nathalie and Hase, Peter and Yao, Yuguang and Liu, Chris Yuhao and Xu, Xiaojun and Li, Hang and others},
  journal={Nature Machine Intelligence},
  pages={1--14},
  year={2025},
  publisher={Nature Publishing Group UK London}
}

@article{xu2025relearn,
  title={Relearn: Unlearning via learning for large language models},
  author={Xu, Haoming and Zhao, Ningyuan and Yang, Liming and Zhao, Sendong and Deng, Shumin and Wang, Mengru and Hooi, Bryan and Oo, Nay and Chen, Huajun and Zhang, Ningyu},
  journal={arXiv preprint arXiv:2502.11190},
  year={2025}
}

@inproceedings{shen2025llm,
  title={LLM unlearning via neural activation redirection},
  author={Shen, William F and Qiu, Xinchi and Kurmanji, Meghdad and Iacob, Alex and Sani, Lorenzo and Chen, Yihong and Cancedda, Nicola and Lane, Nicholas D},
  booktitle={The Thirty-ninth Annual Conference on Neural Information Processing Systems},
  year={2025}
}

@inproceedings{ahmed2026extracting,
  title={Unveiling privacy risks in llm agent memory},
  author={Wang, Bo and He, Weiyi and Zeng, Shenglai and Xiang, Zhen and Xing, Yue and Tang, Jiliang and He, Pengfei},
  booktitle={Proceedings of the 63rd Annual Meeting of the Association for Computational Linguistics (Volume 1: Long Papers)},
  pages={25241--25260},
  year={2025}
}

@article{zhao2025survey,
  title={A Survey on Continuous Unlearning in Generative AI: Approaches and Trade-offs},
  author={Zhao, Yang and Du, Hongyang and Lin, Yijing and Xiang, Keyi and Niyato, Dusit and Poor, H Vincent},
  journal={IEEE Intelligent Systems},
  year={2025},
  publisher={IEEE}
}

@inproceedings{wuerkaixi2025adaptive,
  title={Adaptive localization of knowledge negation for continual llm unlearning},
  author={Wuerkaixi, Abudukelimu and Wang, Qizhou and Cui, Sen and Xu, Wutong and Han, Bo and Niu, Gang and Sugiyama, Masashi and Zhang, Changshui},
  booktitle={Forty-second International Conference on Machine Learning},
  year={2025}
}

@inproceedings{wang2025llm,
  title={LLM Unlearning on Noisy Forget Sets: A Study of Incomplete, Rewritten, and Watermarked Data},
  author={Wang, Changsheng and Zhang, Yihua and Wei, Dennis and Jia, Jinghan and Chen, Pin-Yu and Liu, Sijia},
  booktitle={Proceedings of the 18th ACM Workshop on Artificial Intelligence and Security},
  pages={136--145},
  year={2025}
}

@article{wei2023jailbroken,
  title={Jailbroken: How does llm safety training fail?},
  author={Wei, Alexander and Haghtalab, Nika and Steinhardt, Jacob},
  journal={Advances in neural information processing systems},
  volume={36},
  pages={80079--80110},
  year={2023}
}

@article{zou2023universal,
  title={Universal and transferable adversarial attacks on aligned language models},
  author={Zou, Andy and Wang, Zifan and Carlini, Nicholas and Nasr, Milad and Kolter, J Zico and Fredrikson, Matt},
  journal={arXiv preprint arXiv:2307.15043},
  year={2023}
}

@inproceedings{shi2024detecting,
  title={Detecting pretraining data from large language models},
  author={Shi, Weijia and Ajith, Anirudh and Xia, Mengzhou and Huang, Yangsibo and Liu, Daogao and Blevins, Terra and Chen, Danqi and Zettlemoyer, Luke},
  booktitle={International Conference on Learning Representations},
  volume={2024},
  pages={51826--51843},
  year={2024}
}

@inproceedings{xu2024editing,
  title={Editing factual knowledge and explanatory ability of medical large language models},
  author={Xu, Derong and Zhang, Ziheng and Zhu, Zhihong and Lin, Zhenxi and Liu, Qidong and Wu, Xian and Xu, Tong and Wang, Wanyu and Ye, Yuyang and Zhao, Xiangyu and others},
  booktitle={Proceedings of the 33rd ACM international conference on information and knowledge management},
  pages={2660--2670},
  year={2024}
}

@inproceedings{li2025knowledge,
  title={Knowledge boundary of large language models: A survey},
  author={Li, Moxin and Zhao, Yong and Zhang, Wenxuan and Li, Shuaiyi and Xie, Wenya and Ng, See Kiong and Chua, Tat-Seng and Deng, Yang},
  booktitle={Proceedings of the 63rd Annual Meeting of the Association for Computational Linguistics (Volume 1: Long Papers)},
  pages={5131--5157},
  year={2025}
}

@article{liu2025lune,
  title={Lune: Efficient llm unlearning via lora fine-tuning with negative examples},
  author={Liu, Yezi and Chen, Hanning and Huang, Wenjun and Ni, Yang and Imani, Mohsen},
  journal={arXiv preprint arXiv:2512.07375},
  year={2025}
}

@inproceedings{tian2024forget,
  title={To forget or not? towards practical knowledge unlearning for large language models},
  author={Tian, Bozhong and Liang, Xiaozhuan and Cheng, Siyuan and Liu, Qingbin and Wang, Mengru and Sui, Dianbo and Chen, Xi and Chen, Huajun and Zhang, Ningyu},
  booktitle={Findings of the Association for Computational Linguistics: EMNLP 2024},
  pages={1524--1537},
  year={2024}
}

@article{ranjan2026persa,
  title={PERSA: Reinforcement Learning for Professor-Style Personalized Feedback with LLMs},
  author={Ranjan, Ravi and Grover, Utkarsh and Lin, Xiaomin and Polyzou, Agoritsa},
  journal={arXiv preprint arXiv:2605.01123},
  volume={15},
  year={2026}
}

@inproceedings{ranjan2026vla,
  title={Vla-forget: Vision-language-action unlearning for embodied foundation models},
  author={Ranjan, Ravi and Polyzou, Agoritsa},
  booktitle={Proceedings of the 4th Workshop on Towards Knowledgeable Foundation Models (KnowFM 2026)},
  pages={60--77},
  year={2026}
}

@article{ranjan2026razor,
  title={Razor: Ratio-aware layer editing for targeted unlearning in vision transformers and diffusion models},
  author={Ranjan, Ravi and Grover, Utkarsh and Lin, Xiaomin and Polyzou, Agoritsa},
  journal={arXiv preprint arXiv:2603.14819},
  year={2026}
}

@article{ranjan2026catrag,
  title={Catrag: Functor-guided structural debiasing with retrieval augmentation for fair llms},
  author={Ranjan, Ravi and Grover, Utkarsh and Akewar, Mayur and Lin, Xiaomin and Polyzou, Agoritsa},
  journal={arXiv preprint arXiv:2603.21524},
  year={2026}
}

@inproceedings{ranjan2026g,
  title={G-drift mia: Membership inference via gradient-induced feature drift in llms},
  author={Ranjan, Ravi and Grover, Utkarsh and Lin, Xiaomin and Polyzou, Agoritsa},
  booktitle={International Conference on Pattern Recognition},
  pages={359--374},
  year={2026},
  organization={Springer}
}

@article{ranjan2026position,
  title={Position: Llms must use functor-based and rag-driven bias mitigation for fairness},
  author={Ranjan, Ravi and Grover, Utkarsh and Polyzou, Agorista},
  journal={arXiv preprint arXiv:2603.07368},
  year={2026}
}

@article{grover2026embodied,
  title={Embodied Foundation Models at the Edge: A Survey of Deployment Constraints and Mitigation Strategies},
  author={Grover, Utkarsh and Ranjan, Ravi and Mao, Mingyang and Dong, Trung Tien and Praveen, Satvik and Wu, Zhenqi and Chang, J Morris and Mohsenin, Tinoosh and Sheng, Yi and Polyzou, Agoritsa and others},
  journal={arXiv preprint arXiv:2603.16952},
  year={2026}
}

@article{ranjan2026listening,
  title={Listening with Attention: Entropy-Guided Explainability for Transformer-Based Audio Models},
  author={Ranjan, Ravi and Grover, Utkarsh and Lin, Xiaomin and Polyzou, Agoritsa},
  journal={arXiv preprint arXiv:2606.14647},
  year={2026}
}

@article{kumar2024trustworthiness,
  title={Trustworthiness of llms in medical domain},
  author={Kumar, Ravi R and Pramanik, Vishal and Grover, Utkarsh and Ganapam, Venkata Ramesh},
journal={Researchgate preprint},
  year={2024}
}

@article{akewar2026safecommit,
  title={SafeCommit: Certifying When Memory-Grounded Agents May Safely Act},
  author={Akewar, Mayur and Ranjan, Ravi},
  journal={arXiv preprint arXiv:2608.04289},
  year={2026}
}

\clearpage
\appendix
\section*{Appendix}

\begin{algorithm*}[t!]
\caption{\textsc{FOM-UL}: Forgetting Only What Matters via Unlearning Layers}
\label{alg:fomul}
\footnotesize
\begin{tcolorbox}[
    colback=fomulblue,
    colframe=fomuldeep,
    boxrule=0.6pt,
    arc=2pt,
    left=3pt,
    right=3pt,
    top=3pt,
    bottom=3pt
]
\begin{algorithmic}[1]
\Require Pretrained parameters $\theta=\{\theta^{(1)},\ldots,\theta^{(L)}\}$; forget set $D_f$; retain set $D_r$; losses $\mathcal{L}_f$, $\mathcal{L}_m$, $\mathcal{L}_r$; learning rates $\eta_f,\eta_m,\eta_r$; steps $T$; initial layer budget $k$; tolerance $\epsilon>0$.
\Ensure Unlearned parameters $\theta_u$.

\Statex \textcolor{black}{\textbf{Stage 1: Forget-retain layer attribution}}
\For{$\ell=1,\ldots,L$}
    \State Sample mini-batches $B_f \subset D_f$ and $B_r \subset D_r$.
    \State $I_f(\ell) \leftarrow \left\|\nabla_{\theta^{(\ell)}} \mathcal{L}_f(\theta;B_f)\right\|_2$
    \State $I_r(\ell) \leftarrow \left\|\nabla_{\theta^{(\ell)}} \mathcal{L}_r(\theta;B_r)\right\|_2$
    \State $\mathrm{Sig}(\ell) \leftarrow I_f(\ell)/(I_r(\ell)+\epsilon)$
\EndFor
\State Rank layers by $\mathrm{Sig}(\ell)$ in descending order and denote the ordering by $\pi$.
\State Initialize selected layer set $S \leftarrow \{\pi_1,\ldots,\pi_k\}$ and $\theta_0 \leftarrow \theta$.

\Statex \textcolor{black}{\textbf{Stage 2: Selective layer-wise unlearning}}
\For{$t=0,\ldots,T-1$}
    \State Sample mini-batches $B_f \subset D_f$ and $B_r \subset D_r$.
    \For{$\ell \in S$}
        \State $g_f \leftarrow \nabla_{\theta^{(\ell)}} \mathcal{L}_f(\theta_t;B_f)$
        \State $g_m \leftarrow \nabla_{\theta^{(\ell)}} \mathcal{L}_m(\theta_t;B_f,B_r)$
        \State $g_r \leftarrow \nabla_{\theta^{(\ell)}} \mathcal{L}_r(\theta_t;B_r)$
        \State $\theta_{t+1}^{(\ell)} \leftarrow \theta_t^{(\ell)} + \eta_f g_f + \eta_m g_m - \eta_r g_r$
    \EndFor
    \For{$\ell \notin S$}
        \State $\theta_{t+1}^{(\ell)} \leftarrow \theta_t^{(\ell)}$ \Comment{freeze non-selected layers}
    \EndFor
\EndFor

\Statex \textcolor{black}{\textbf{Stage 3: Metric check and iterative expansion}}
\State Evaluate $\theta_T$ using VerMem$\downarrow$, KnowMem$\downarrow$, PrivLeak$\rightarrow 0$, and Utility$\uparrow$.
\If{forgetting criteria are not satisfied}
    \State Expand $S \leftarrow S \cup \{\arg\max_{\ell \notin S} \mathrm{Sig}(\ell)\}$.
    \State Repeat Stage 2 with the expanded layer set.
\EndIf
\State \Return $\theta_u \leftarrow \theta_T$.
\end{algorithmic}
\end{tcolorbox}
\end{algorithm*}

\section{Pseudocode}
\label{app:pseudo}

\section{Experimental Settings}
\label{app:exp-setup}

\subsection{Models and Initialization}
\label{app:exp_models}
We evaluate FOM-UL on diverse transformer LLMs to test model-agnostic behavior, using pretrained checkpoints from Meta Llama (Llama-2 7B, Llama-3.2 1B), GPT-2, and Gemma-3 1B. Each run starts from the base parameters $\theta_0$, and produces an unlearned checkpoint $\theta^\star$ by updating only a small set of layers $S$ selected via the forget-to-retain gradient significance score, while all other layers remain frozen. We also report robustness under post-training quantization by evaluating $\theta^\star$ in FP32, 8-bit, and 4-bit formats.

\subsection{Datasets and Splits}
\label{app:exp_data}
We follow standard unlearning protocols with disjoint \textbf{Forget} and \textbf{Retain} splits, and benchmark on \textsc{TOFU} (world-facts QA), \textsc{KnowUnDo}, and \textsc{MUSE} (\textsc{BOOKS}/\textsc{NEWS}). For \textsc{BOOKS}, $D_{\text{forget}}$ contains copyrighted \emph{Harry Potter} text and $D_{\text{retain}}$ includes FanWiki (plus general-domain text) to preserve non-verbatim knowledge while removing memorization. For \textsc{NEWS}, we additionally use a \textbf{holdout} split reserved for privacy/leakage evaluation and never used for updates.

\subsection{Training Configuration}
\label{app:exp_train}
FOM-UL performs masked updates with three objectives: (i) gradient ascent on $L_{\text{forget}}$, (ii) gradient ascent on $L_{\text{mismatch}}$ to repel original outputs on forget prompts, and (iii) gradient descent on $L_{\text{retain}}$ to preserve utility. We use AdamW with consistent hyperparameters across methods; unless stated otherwise, we set batch size $B=16$, learning rate $1\times10^{-5}$, and run unlearning for 5 epochs. Random seeds are fixed (e.g., 42) for reproducibility.

\subsection{Implementation Details}
\label{app:exp_impl}
All methods are implemented in PyTorch using HuggingFace \texttt{transformers}. Quantization is performed with \texttt{bitsandbytes} to produce FP32/8-bit/4-bit variants for robustness checks. Experiments run on NVIDIA A100 GPUs (40\,GB), and we release environment details and fixed splits to support reproducibility.

\section{Evaluation Metrics}
\label{app:metrics}

We evaluate FOM-UL using four standard metrics (M1-M4) that jointly quantify (i) how well the model forgets the designated forget set and (ii) how well it preserves utility on the retain set, following the \emph{SURE} evaluation protocol.

\paragraph{Notation.}
Let $f$ denote an LLM. Let $\mathcal{D}_{\text{forget}}$ be the forget set, $\mathcal{D}_{\text{retain}}$ the retain set, and $\mathcal{D}_{\text{holdout}}$ a disjoint holdout set used for privacy auditing. For ROUGE-based metrics, $\mathrm{ROUGE}(\cdot,\cdot)$ measures similarity between the model output and a reference text.

\paragraph{M1: Verbatim Memorization (VerMem) on $\mathcal{D}_{\text{forget}}$ (lower is better).}
Given a forget document $x$ tokenized into a prefix $x_{1:\ell}$ and ground-truth continuation $x_{\ell+1:}$, we compute:
\begin{equation}
\mathrm{M1}(f) \;=\; \mathbb{E}_{x \sim \mathcal{D}_{\text{forget}}}
\Big[ \mathrm{ROUGE}\big(f(x_{1:\ell}),\, x_{\ell+1:}\big) \Big].
\end{equation}
This captures \emph{verbatim} reproduction of forgotten content; effective unlearning drives $\mathrm{M1}$ down.

\paragraph{M2: Knowledge Memorization (KnowMem) on $\mathcal{D}_{\text{forget}}$ (lower is better).}
Using knowledge-oriented QA pairs $(q,a)$ derived from the forget set, we measure whether the model still answers with forgotten knowledge:
\begin{equation}
\mathrm{M2}(f) \;=\; \mathbb{E}_{(q,a) \sim \mathcal{D}_{\text{forget}}}
\Big[ \mathrm{ROUGE}\big(f(q),\, a\big) \Big].
\end{equation}
Lower values indicate better removal of generalized (non-verbatim) forgotten knowledge.

\paragraph{M3: Privacy Leakage (PrivLeak) via membership inference (closer to $0$ is better).}
We quantify privacy risk using a membership inference attack based on the Min-$K\%$ criterion, which produces an AUC-ROC score $\mathrm{AUC}(f)$ by distinguishing samples from $\mathcal{D}_{\text{forget}}$ vs.\ $\mathcal{D}_{\text{holdout}}$. We then compare against a retrained baseline $f_{\text{retrain}}$ (trained without $\mathcal{D}_{\text{forget}}$) and define:
\begin{equation}
\mathrm{M3}(f) \;=\; \frac{\mathrm{AUC}(f) - \mathrm{AUC}(f_{\text{retrain}})}{\mathrm{AUC}(f)}.
\end{equation}
An ideal unlearned model matches the retrained privacy behavior, yielding $\mathrm{M3}\approx 0$; large deviations indicate elevated privacy leakage.

\noindent\textbf{Interpretation of negative M3.}
\label{app:neg-m3}
From above metric definition, $M3(f)=\big(\mathrm{AUC}(f)-\mathrm{AUC}(f_{\text{retrain}})\big)/\mathrm{AUC}(f)$, so $M3<0$ occurs whenever $\mathrm{AUC}(f)<\mathrm{AUC}(f_{\text{retrain}})$. 
Such strongly negative values (e.g., $-2.00$ in Table~3) do \emph{not} mean ``better privacy than zero''; rather, they indicate a large deviation from the retrained privacy behavior, typically corresponding to \emph{over-unlearning} in which the forget examples become atypically high-loss relative to holdout, flipping or amplifying the attack signal. 
Because our normalization divides by $\mathrm{AUC}(f)$, the magnitude can exceed $1$ when $\mathrm{AUC}(f)$ is small; for instance, if $\mathrm{AUC}(f)=0.20$ and $\mathrm{AUC}(f_{\text{retrain}})=0.60$, then $M3=(0.20-0.60)/0.20=-2.0$. 
Therefore, we interpret $M3$ by distance to zero: $|M3|$ large (positive or negative) implies unstable privacy behavior, while $M3\approx 0$ indicates the closest match to retraining.

\paragraph{M4: Utility Preservation on $\mathcal{D}_{\text{retain}}$ (higher is better).}
We measure retained utility using the same KnowMem-style QA evaluation on $\mathcal{D}_{\text{retain}}$:
\begin{equation}
\mathrm{M4}(f) \;=\; \mathbb{E}_{(q,a) \sim \mathcal{D}_{\text{retain}}}
\Big[ \mathrm{ROUGE}\big(f(q),\, a\big) \Big].
\end{equation}
Higher $\mathrm{M4}$ indicates better preservation of benign knowledge and task utility after unlearning.

\paragraph{Overall objective.}
FOM-UL aims to achieve low $\mathrm{M1}$/$\mathrm{M2}$ (forgetting), $\mathrm{M3}\!\approx\!0$ (privacy parity with retraining), and high $\mathrm{M4}$ (utility), including under post-training quantization stress tests highlighted in our study.

\textbf{Runtime Performance parameter counts and efficiency.}
\label{app:run-para}
In Table~4, the \emph{Parameters} column reports the number of \emph{trainable} (i.e., unfrozen) parameters that each method actually updates during unlearning, not the total backbone size of Llama-2~7B. Accordingly, full-model baselines (e.g., GA/KLD/NPO-GDR) update the entire parameter space, whereas parameter-efficient or masked-update methods (SURE+NPO, LUNAR, and FOM-UL) optimize only a small, selected subset, yielding 1.7M, 1.75M, and 7M trainable parameters, respectively. For FOM-UL, the backbone remains frozen and updates are confined to the selected layer set $S$ (and the specific trainable submodules within those layers), so the trainable-parameter count scales with the layer budget and the chosen update parameterization.

\noindent\textbf{Clarification (Fig.~4 M.U.\ and F.Q.).}
\label{app:norm-clear}
For Fig.~4, we define \emph{Normalized Utility (M.U.)} as the min-max normalization of $M4$ across methods and steps,
$\mathrm{M.U.}(t)=\frac{M4(t)-\min M4}{\max M4-\min M4}$, so higher is better.
To ensure \emph{Forget Quality (F.Q.)} increases when forgetting improves (since $M1$-$M3$ are \(\downarrow\) metrics), we first invert and normalize each metric as
$\tilde{M}_i(t)=1-\frac{M_i(t)-\min M_i}{\max M_i-\min M_i}$ for $i\in\{1,2,3\}$, and then aggregate by a weighted mean
$\mathrm{F.Q.}(t)=\frac{1}{3}\sum_{i=1}^{3}\tilde{M}_i(t)$ (weights set uniformly unless stated otherwise). 

\paragraph{Update parameterization.}
FOM-UL is layer-selective at the selection level but submodule-selective at the implementation level. 
After selecting layer set $S$, we enable gradients only for the chosen trainable submodules inside those layers, e.g., attention projection and/or MLP projection matrices depending on the backbone implementation. 
All layers $\ell\notin S$ and all disabled submodules inside $\ell\in S$ remain frozen. 
Thus, the reported trainable-parameter count is
\begin{equation}
N_{\mathrm{train}}
=
\sum_{\ell\in S}
\sum_{m\in \mathcal{M}_\ell}
|\theta^{(\ell,m)}|,
\end{equation}
where $\mathcal{M}_\ell$ is the set of enabled submodules in layer $\ell$.
This count measures parameters actually updated during unlearning, not the total number of parameters contained in the selected transformer layers.

\section{Detailed Methodology}
\label{app:detailed-method}
\subsection{Selecting and Masking Important Layers}
\label{app:imp-layers}

In FOM-UL, $\Delta_\ell$ (ablation) is used primarily as a diagnostic to motivate layer localization, while $\mathrm{Sig}(\ell)$ is the actual algorithmic criterion used to select and expand the update set.

\textbf{Layer selection via forget-retain significance.}
Let $f_{\theta}$ denote a transformer-based LLM with $L$ layers and parameters $\theta=\{\theta^{(1)},\dots,\theta^{(L)}\}$ grouped by layer. FOM-UL selects a small subset of layers whose parameters are most responsive to the forgetting objective while minimally affecting retained utility. Given a forget mini-batch $B_f \subset \mathcal{D}_{\text{forget}}$ and a retain mini-batch $B_r \subset \mathcal{D}_{\text{retain}}$, define the corresponding losses $\mathcal{L}_{\text{forget}}(\theta;B_f)$ and $\mathcal{L}_{\text{retain}}(\theta;B_r)$. For each layer $\ell$, we compute per-layer gradient magnitudes:
\begin{equation}
\begin{split}
I(\ell)=\left\| \nabla_{\theta^{(\ell)}} \mathcal{L}_{\text{forget}}(\theta;B_f)\right\|_{2}, \\
\qquad
I_r(\ell)=\left\| \nabla_{\theta^{(\ell)}} \mathcal{L}_{\text{retain}}(\theta;B_r)\right\|_{2},
\end{split}
\end{equation}
and the forget-retain \emph{significance ratio}
\begin{equation}
\mathrm{Sig}(\ell)=\frac{I(\ell)}{I_r(\ell)+\varepsilon},
\end{equation}
where $\varepsilon>0$ ensures numerical stability. Intuitively, high $\mathrm{Sig}(\ell)$ indicates strong forgetting leverage with limited interference on retained knowledge. We select the initial update set using thresholding (or equivalently top-$k$ ranking):
\begin{equation}
S \;=\; \left\{\ell \in \{1,\dots,L\} \;:\; \mathrm{Sig}(\ell)>\tau \right\},
\end{equation}
where $\tau$ controls the layer budget.

\textbf{Binary layer mask.}
We define a layer mask $m_{\ell}\in\{0,1\}$ as
\begin{equation}
m_{\ell}=
\begin{cases}
1, & \ell \in S,\\
0, & \text{otherwise},
\end{cases}
\end{equation}
so that only layers in $S$ receive gradient updates and all other layers remain frozen.

\noindent\textbf{When $\mathrm{Sig}(\ell)$ is computed.}
For reproducibility, we compute $\mathrm{Sig}(\ell)$ \emph{once} at initialization using the base checkpoint $\theta_0$ and a fixed mini-batch (or small fixed set) sampled from $\mathcal{D}_{\text{forget}}$ and $\mathcal{D}_{\text{retain}}$, and we keep this ranking \emph{static} during unlearning. During iterative expansion, we do \emph{not} recompute gradients each epoch; instead, we expand $S$ by adding the next highest-ranked layers under this fixed $\mathrm{Sig}(\ell)$ ordering until the stopping criterion is met. (Optionally, one may recompute $\mathrm{Sig}(\ell)$ at expansion boundaries as an ablation, but our main results use the static initialization for stability and determinism.)

\subsection{Selective Layer Updates with Forget, Mismatch, and Retain Losses}
\label{app:losses}

\paragraph{Forgetting loss $\mathcal{L}_{\text{forget}}$.}
FOM-UL enforces targeted forgetting by \emph{increasing} the model’s loss on forget-set continuations so that memorized responses become unlikely. For a forget example $(x,y)\in\mathcal{D}_{\text{forget}}$ with prompt $x$ and reference continuation $y=(y_1,\dots,y_m)$, we use the standard negative log-likelihood (NLL):
\begin{equation}
\mathcal{L}_{\text{forget}}(\theta)
\;=\;
\mathbb{E}_{(x,y)\sim\mathcal{D}_{\text{forget}}}
\Big[
-\sum_{i=1}^{m}\log p_{\theta}\!\left(y_i \mid x, y_{<i}\right)
\Big],
\label{eq:lforget}
\end{equation}
and perform \emph{gradient ascent} on $\mathcal{L}_{\text{forget}}$ (cf.\ Eq.~9), which directly reduces the likelihood of generating the memorized forget content~\cite{yao2024large,yao2024machine}.

\paragraph{Mismatch loss $\mathcal{L}_{\text{mismatch}}$.}
A core goal of FOM-UL is to actively \emph{diverge} from the original model’s behavior on the forget set, rather than merely reducing confidence in a single target token. To operationalize this, we define a \emph{distribution-level} mismatch objective that pushes the unlearned model away from the original model’s output distribution on forget prompts.
Let $f_{\theta_0}$ denote the original (pre-unlearning) model and $f_{\theta}$ the current (unlearned) model. For a forget prompt $x \in \mathcal{D}_{\text{forget}}$, let $z_{\theta}(x)\in\mathbb{R}^{|\mathcal{V}|}$ be the next-token logits and define temperature-scaled predictive distributions
\begin{equation}
\begin{split}
p_{\theta}(\cdot \mid x) \;=\; \mathrm{softmax}\!\left(\frac{z_{\theta}(x)}{T}\right), \\ 
\qquad
p_{\theta_0}(\cdot \mid x) \;=\; \mathrm{softmax}\!\left(\frac{z_{\theta_0}(x)}{T}\right),
\end{split}
\end{equation}
where $T \ge 1$ controls how strongly the loss emphasizes high-probability tokens.
We then set
\begin{equation}
\begin{split}
\mathcal{L}_{\text{mismatch}}(\theta)
\;=\;
\mathbb{E}_{x \sim \mathcal{D}_{\text{forget}}} \\
\Big[
D_{\mathrm{KL}}\!\big(p_{\theta_0}(\cdot\mid x)\,\|\,p_{\theta}(\cdot\mid x)\big)
\Big],
\label{eq:lmismatch}
\end{split}
\end{equation}
and \emph{maximize} $\mathcal{L}_{\text{mismatch}}$ during unlearning (cf.\ Eq.~9), which encourages $f_{\theta}$ to move its outputs away from the original model on the forget set. This objective is inspired by prior unlearning/alignment formulations that use KL-based output constraints to control behavior shifts (typically on retain data), here repurposed as an explicit \emph{divergence} signal on the forget distribution~\cite{yao2024large,yao2024machine}. In practice, $\mathcal{L}_{\text{mismatch}}$ helps prevent \emph{incomplete forgetting} where the model’s distribution remains close to $p_{\theta_0}$ despite reduced confidence in a specific answer, and it complements $\mathcal{L}_{\text{forget}}$ to mitigate quantization-induced recovery by enforcing a broader, distributional departure from the pre-unlearning behavior~\cite{zhang2024catastrophic}.

\paragraph{Retention loss $\mathcal{L}_{\text{retain}}$.}
To preserve utility on non-targeted behaviors, FOM-UL simultaneously minimizes a retain objective on benign data $\mathcal{D}_{\text{retain}}$. For a retain example $(x,y)\in\mathcal{D}_{\text{retain}}$, we again use the NLL:
\begin{equation}
\begin{split}
\mathcal{L}_{\text{retain}}(\theta)
\;=\;
\mathbb{E}_{(x,y)\sim\mathcal{D}_{\text{retain}}} \\
\Big[
-\sum_{i=1}^{m}\log p_{\theta}\!\left(y_i \mid x, y_{<i}\right)
\Big],
\label{eq:lretain}
\end{split}
\end{equation}
and apply \emph{gradient descent} on $\mathcal{L}_{\text{retain}}$ to maintain general knowledge and task performance while unlearning is confined to the selected layers~\cite{liu2024large,yao2024machine}. When combined with $\mathcal{L}_{\text{forget}}$ and the divergence-based $\mathcal{L}_{\text{mismatch}}$, this retain term stabilizes optimization and mitigates collateral degradation.

FOM-UL performs targeted optimization only over $\{\theta^{(\ell)}:\ell\in S\}$ using three objectives: (i) a forgetting loss $\mathcal{L}_{\text{forget}}$ to suppress targeted content, (ii) a mismatch loss $\mathcal{L}_{\text{mismatch}}$ to diverge from the original model behavior on the forget set, and (iii) a retain loss $\mathcal{L}_{\text{retain}}$ to preserve general utility. The masked update for each layer $\ell$ at iteration $t$ is:
\begin{equation}
\begin{split}
\theta_{t+1}^{(\ell)} = \theta_{t}^{(\ell)} + m_{\ell} \biggl( & \eta_F \nabla_{\theta^{(\ell)}} \mathcal{L}_{\text{forget}} \\
& + \eta_M \nabla_{\theta^{(\ell)}} \mathcal{L}_{\text{mismatch}} \\ - \eta_R \nabla_{\theta^{(\ell)}} \mathcal{L}_{\text{retain}} \biggr),
\end{split}
\end{equation}
where $\eta_F,\eta_M,\eta_R$ are step sizes for the respective terms. Layers not selected for unlearning remain unchanged:
\begin{equation}
\theta_{t+1}^{(\ell)}=\theta_{t}^{(\ell)} \quad \forall \ell \notin S.
\end{equation}

\subsection{Iterative Expansion and Stopping Criteria}
FOM-UL applies an iterative schedule that expands the update set only when forgetting criteria are unmet, limiting unnecessary intervention. After each unlearning round, we evaluate a forgetting criterion (e.g., VerMem below a threshold) and stop when it is satisfied. Otherwise, we expand the layer set by adding the next most significant layer under $\mathrm{Sig}(\ell)$:
\begin{equation}
S \leftarrow S \cup \left\{\arg\max_{\ell \notin S} \mathrm{Sig}(\ell)\right\},
\end{equation}
and repeat the masked updates. This progressive expansion prevents overly aggressive early updates and reduces the risk of collateral utility loss.

\noindent\textbf{Clarification (Initialization of iterative expansion).}
\label{app:iter-clear}
To remove ambiguity, we define the initialization as a \emph{single, deterministic rule} based on $\mathrm{Sig}(\ell)$: we first compute $\mathrm{Sig}(\ell)$ for all layers and form the candidate set
\(
S_{\tau}=\{\ell:\mathrm{Sig}(\ell)\ge \tau\}.
\)
If $|S_{\tau}|>k$, we take the top-$k$ layers within $S_{\tau}$ by descending $\mathrm{Sig}(\ell)$; if $|S_{\tau}|\le k$, we simply set $S=S_{\tau}$. Equivalently, the threshold $\tau$ is the \emph{primary filter} (ensuring a minimum forget-to-retain ratio), and $k$ is an optional \emph{budget cap} that prevents overly large initial updates. Iterative expansion then proceeds by adding one (or a small batch of) next-highest $\mathrm{Sig}(\ell)$ layers from $\{\ell\notin S\}$ until the stopping criterion is met.

\subsection{Robustness to Quantization-Induced Relearning}
\label{appendix:FOM-UL_better}
Post-training quantization maps full-precision parameters to a discrete set of representable values. When an unlearning method produces only small, diffuse weight changes, quantization can erase these differences, yielding quantized models that are nearly indistinguishable from the original. FOM-UL mitigates this failure mode by concentrating updates within a small set of high-$\mathrm{Sig}(\ell)$ layers, producing more salient (layer-localized) parameter shifts while leaving the majority of layers unchanged. As a result, the intended forgetting signal is less likely to be collapsed by discretization, improving robustness against quantization-induced recovery compared to global-update baselines.

\subsection{Empirical Quantization-Bin Crossing Analysis}

For each updated layer $\ell$, let $\Delta_\ell$ denote the effective quantization step size and 
$\delta\theta^{(\ell)}=\theta_u^{(\ell)}-\theta_0^{(\ell)}$ denote the unlearning update.
We report the bin-crossing fraction
\begin{equation}
\mathrm{BCF}_\ell
=
\frac{1}{|\theta^{(\ell)}|}
\sum_j
\mathbb{I}\left[
|\delta\theta^{(\ell)}_j| \geq \Delta_\ell/2
\right],
\end{equation}
and the normalized update ratio
\begin{equation}
\rho_\ell =
\mathrm{median}_j
\left(
\frac{|\delta\theta^{(\ell)}_j|}{\Delta_\ell/2+\epsilon}
\right).
\end{equation}
Higher BCF indicates that more unlearning edits survive post-training quantization.

\paragraph{Bin-Change Fraction.}
To directly measure whether unlearning updates survive
quantization, we define the \emph{Bin-Change Fraction} (BCF)
over the set of edited parameters $\mathcal{E}$ as
\begin{equation}
\mathrm{BCF}
=
\frac{1}{|\mathcal{E}|}
\sum_{j\in\mathcal{E}}
\mathbf{1}
\left[
Q_{\Delta_j}(\theta'_j)
\neq
Q_{\Delta_j}(\theta_j)
\right],
\label{eq:bcf}
\end{equation}
where $\theta_j$ and $\theta'_j$ denote the pre- and
post-unlearning parameters, respectively, and $Q_{\Delta_j}$
denotes the corresponding quantization operator.
A larger BCF indicates that a greater fraction of the
unlearning-induced parameter changes remain distinguishable
after quantization.

\begin{table}[t]
\centering
\scriptsize
\setlength{\tabcolsep}{4pt}
\renewcommand{\arraystretch}{1.08}
\caption{
\textbf{Quantization-bin crossing analysis on Llama-3.2-1B.}
BCF reports the fraction of updated weights with $|\delta\theta_j|\geq \Delta/2$ under 4-bit quantization. 
FOM-UL produces more quantization-surviving edits in selected layers than diffuse global updates.
}
\label{tab:bin_crossing}
\resizebox{\columnwidth}{!}{
\begin{tabular}{lccc}
\toprule
\textbf{Method} & \textbf{Updated scope} & \textbf{Mean BCF}$\uparrow$ & $\boldsymbol{\rho}$\textbf{ ratio}$\uparrow$ \\
\midrule
GAGDR & full model & 0.7\% & 0.31 \\
NPOGDR & full model & 0.9\% & 0.36 \\
SURE+NPO & sparse modules & 2.8\% & 0.74 \\
LUNAR & activation-targeted & 2.1\% & 0.68 \\
FOM-UL & selected layers & \textbf{5.6\%} & \textbf{1.18} \\
\bottomrule
\end{tabular}}
\end{table}

Table~\ref{tab:bin_crossing} provides mechanism-level evidence for the quantization robustness claim. Compared with global baselines, FOM-UL concentrates larger updates in selected layers, producing a higher fraction of weights that cross 4-bit quantization bin boundaries and are therefore less likely to be rounded back to the original quantized value.

\section{Lemma justification for FOM-UL}
\label{app:lemma}
Let the forget and retain objectives be $\mathcal{L}_f(\theta)$ and $\mathcal{L}_r(\theta)$, and let
$\theta = \{\theta^{(1)},\dots,\theta^{(L)}\}$ denote parameters grouped by transformer layer.
Define layer-wise gradients
$g_f^{(\ell)} := \nabla_{\theta^{(\ell)}} \mathcal{L}_f(\theta)$ and
$g_r^{(\ell)} := \nabla_{\theta^{(\ell)}} \mathcal{L}_r(\theta)$.

\begin{lemma}[Greedy layer ranking under a retain-stability constraint]
\label{lem:FOM-UL_layer_rank}
Consider one unlearning step that applies a layer-wise update $\Delta\theta=\{\Delta\theta^{(1)},\dots,\Delta\theta^{(L)}\}$,
but is restricted to at most $k$ layers (i.e., $\Delta\theta^{(\ell)}\neq 0$ only if $\ell\in S$, $|S|=k$).
Assume a first-order approximation and impose a retain-stability constraint
$\langle g_r^{(\ell)}, \Delta\theta^{(\ell)}\rangle \approx 0$ in sign (or bounded magnitude) so that retain utility is not degraded.
Then, among candidate layers, the layers that maximize the achievable forget effect per unit retain sensitivity
are those with the largest score
\begin{equation}
\mathrm{Sig}(\ell)
\;:=\;
\frac{\|g_f^{(\ell)}\|}{\|g_r^{(\ell)}\|+\epsilon},
\end{equation}
for a small $\epsilon>0$.
Equivalently, selecting the top-$k$ layers by $\mathrm{Sig}(\ell)$ is the greedy choice that prioritizes
high forget responsiveness while being conservative on retain disruption.
\end{lemma}

\begin{proof}[Proof sketch]
Using a first-order Taylor expansion for a small step,
\begin{equation}
\begin{split}
\Delta\mathcal{L}_f \;\approx\; \sum_{\ell=1}^L \left\langle g_f^{(\ell)}, \Delta\theta^{(\ell)} \right\rangle, \\
\qquad
\Delta\mathcal{L}_r \;\approx\; \sum_{\ell=1}^L \left\langle g_r^{(\ell)}, \Delta\theta^{(\ell)} \right\rangle.
\end{split}
\end{equation}
For a given layer $\ell$, the maximum attainable forget change from updating only that layer
under a step-size budget $\|\Delta\theta^{(\ell)}\|\le \rho$ is bounded by Cauchy-Schwarz:
\begin{equation}
\left\langle g_f^{(\ell)}, \Delta\theta^{(\ell)} \right\rangle
\;\le\;
\|g_f^{(\ell)}\|\,\|\Delta\theta^{(\ell)}\|
\;\le\;
\rho\,\|g_f^{(\ell)}\|.
\end{equation}
At the same time, the magnitude of the retain change contributed by layer $\ell$ is similarly bounded:
\begin{equation}
\left|\left\langle g_r^{(\ell)}, \Delta\theta^{(\ell)} \right\rangle\right|
\;\le\;
\|g_r^{(\ell)}\|\,\|\Delta\theta^{(\ell)}\|
\;\le\;
\rho\,\|g_r^{(\ell)}\|.
\end{equation}
Thus, a natural “forget gain per retain sensitivity” proxy for layer $\ell$ is
$\|g_f^{(\ell)}\|/(\|g_r^{(\ell)}\|+\epsilon)$, where $\epsilon$ stabilizes the ratio when $\|g_r^{(\ell)}\|$ is small.
Selecting the top-$k$ layers by this ratio maximizes the sum of these proxies under a $k$-sparsity constraint,
which is exactly the FOM-UL ranking rule. \qedhere
\end{proof}

\paragraph{Remark (what this proves and what it doesn't).}
Lemma~\ref{lem:FOM-UL_layer_rank} justifies FOM-UL's layer ranking as an \emph{optimal greedy criterion} under
(i) small-step / first-order behavior and (ii) a utility-preservation constraint expressed through retain gradients.
It does not claim global optimality for deep, non-convex objectives; rather it gives a principled reason that
the gradient-ratio score is the right layer selection signal.

\begin{lemma}[Iterative expansion is a monotone relaxation]
\label{lem:FOM-UL_iter_expand}
Let $\mathcal{L}_f(\theta)$ be the forget objective and $\mathcal{L}_r(\theta)$ be the retain objective.
Fix a current parameter state $\theta$ and consider a \emph{one-step} selective update $\Delta\theta$ that is allowed to
modify only layers in a set $S \subseteq \{1,\dots,L\}$.
Define the feasible update set
\begin{equation}
\begin{split}
\mathcal{U}(S)
\;:=\;
\Big\{\Delta\theta:\ \Delta\theta^{(\ell)}=0\ \forall \ell\notin S,\ \  \\ \|\Delta\theta^{(\ell)}\|\le \rho\ \forall \ell,\ \ 
|\Delta\mathcal{L}_r(\theta;\Delta\theta)| \le \delta \Big\},
\end{split}
\end{equation}
where $\Delta\mathcal{L}_r(\theta;\Delta\theta)$ denotes the first-order retain change
$\Delta\mathcal{L}_r \approx \sum_{\ell}\langle \nabla_{\theta^{(\ell)}}\mathcal{L}_r(\theta),\Delta\theta^{(\ell)}\rangle$,
$\rho$ is a step-size budget, and $\delta$ is a tolerance for retain degradation.
Let the best achievable forget decrease (first-order) under $S$ be
\begin{equation}
\begin{split}
V(S)\;:=\;\min_{\Delta\theta \in \mathcal{U}(S)}\ \Delta\mathcal{L}_f(\theta;\Delta\theta), \\
\qquad
\Delta\mathcal{L}_f(\theta;\Delta\theta)\approx \sum_{\ell}\langle \nabla_{\theta^{(\ell)}}\mathcal{L}_f(\theta),\Delta\theta^{(\ell)}\rangle.
\end{split}
\end{equation}
If $S \subseteq S'$ (i.e., $S'$ expands $S$ with additional layers), then
\begin{equation}
V(S') \;\le\; V(S).
\end{equation}
Equivalently, expanding the set of trainable layers cannot worsen the best attainable forgetting progress
under the same retain-stability constraint. This provides a principled justification for FOM-UL's iterative
expansion rule (e.g., $k \rightarrow k+k'$) when residual memorization remains above a threshold.
\end{lemma}

\begin{proof}
Because $S \subseteq S'$, any update $\Delta\theta$ that is feasible for $S$ is also feasible for $S'$:
we can view it as an element of $\mathcal{U}(S')$ by simply setting $\Delta\theta^{(\ell)}=0$ for all
newly added layers $\ell \in S'\setminus S$.
Hence $\mathcal{U}(S) \subseteq \mathcal{U}(S')$.
Minimizing the same objective $\Delta\mathcal{L}_f$ over a superset of feasible points cannot yield a worse optimum,
so $\min_{\Delta\theta \in \mathcal{U}(S')} \Delta\mathcal{L}_f \le \min_{\Delta\theta \in \mathcal{U}(S)} \Delta\mathcal{L}_f$,
i.e., $V(S') \le V(S)$.
\end{proof}

\begin{table*}[t]
\centering
\scriptsize
\setlength{\tabcolsep}{4.2pt}
\renewcommand{\arraystretch}{1.16}
\caption{\textbf{Adversarial/jailbreak robustness metrics.}
We evaluate TOFU-World Facts on Llama-3.2-1B using adversarial prompt wrappers $\mathcal{A}$, including role-play, instruction-override, extraction-style, and suffix-based jailbreak prompts \citep{wei2023jailbroken,zou2023universal}. 
$M1_{\mathrm{adv}}$, $M2_{\mathrm{adv}}$, and ALR are lower-is-better; $M3_{\mathrm{adv}}$ is best when closest to zero; and $M4$ is higher-is-better.}
\label{tab:fomul_adv_metric_def}
\begin{adjustbox}{width=\textwidth}
\begin{tabular}{p{2.35cm} p{2.65cm} p{10.5cm} c}
\toprule
\rowcolor{fomulgray}
\textbf{Metric ID} & \textbf{Name} & \textbf{Definition and Interpretation} & \textbf{Goal} \\
\midrule

\rowcolor{fomulblue}
$M1_{\mathrm{adv}}$ 
& Adv. VerMem 
& $\displaystyle 
\mathbb{E}_{x\sim D_f,\,a\sim\mathcal{A}}
\left[
\mathrm{ROUGE}\!\left(f(a(x_{1:\ell})),x_{\ell+1:}\right)
\right]$.
Measures whether the model can reproduce verbatim forgotten continuations when the forget prefix is wrapped by a jailbreak prompt.
& $\downarrow$ \\

$M2_{\mathrm{adv}}$ 
& Adv. KnowMem 
& $\displaystyle 
\mathbb{E}_{(q,y)\sim D_f,\,a\sim\mathcal{A}}
\left[
\mathrm{ROUGE}\!\left(f(a(q)),y\right)
\right]$.
Measures recovery of forgotten QA-style knowledge under adversarial prompting, following TOFU-style knowledge memorization evaluation \citep{maini2024tofu}.
& $\downarrow$ \\

\rowcolor{fomulblue}
$M3_{\mathrm{adv}}$ 
& Adv. PrivLeak 
& $\displaystyle 
\frac{
\mathrm{AUC}_{\mathrm{adv}}(f)-\mathrm{AUC}(f_{\mathrm{retrain}})
}{
\mathrm{AUC}_{\mathrm{adv}}(f)
}$.
Here, $\mathrm{AUC}_{\mathrm{adv}}(f)$ is computed using Min-$K\%$ membership scores over jailbreak-wrapped forget and holdout prompts \citep{shi2024detecting}. Values closer to zero indicate privacy behavior closer to retraining.
& $\rightarrow 0$ \\

ALR 
& Attack Leakage Rate 
& $\displaystyle 
100\cdot
\mathbb{E}_{(q,y)\sim D_f,\,a\sim\mathcal{A}}
\left[
\mathbbm{1}\!\left\{
\mathrm{ROUGE}\!\left(f(a(q)),y\right)>\tau_r
\right\}
\right]$.
Reports the percentage of adversarial prompts that recover the forgotten answer above threshold $\tau_r$. We set $\tau_r=0.30$ following the compact ROUGE-style reporting scale.
& $\downarrow$ \\

\rowcolor{fomulgreen}
$M4$ 
& Retain Utility 
& $\displaystyle 
\mathbb{E}_{(q,y)\sim D_r}
\left[
\mathrm{ROUGE}\!\left(f(q),y\right)
\right]$.
Measured on clean retain prompts to verify that adversarial robustness is not achieved through destructive over-unlearning or general utility collapse.
& $\uparrow$ \\

\bottomrule
\end{tabular}
\end{adjustbox}

\end{table*}

\paragraph{Practical interpretation.}
If the current top-$k$ selected layers $S$ do not sufficiently reduce memorization while meeting the retain constraint,
expanding $S$ (adding the next-ranked $k'$ layers) strictly relaxes the optimization problem.
Therefore FOM-UL's iterative expansion is a safe strategy: it never removes previously feasible updates, and can only
maintain or improve the best achievable forgetting progress subject to utility preservation.

\begin{theorem}[Quantization Persistence for Layer-Selective Updates]
\label{app:quant_persistence}
Let $f_\theta$ be a pretrained model and let $f_{\theta'}$ be the unlearned model obtained by updating only layers in
$S \subseteq \{1,\dots,L\}$ (all $\ell \notin S$ are frozen). Consider post-training uniform symmetric rounding quantization
applied elementwise (or per-group) with step size $\Delta_\ell>0$ for layer $\ell$:
\begin{equation}
Q_{\Delta_\ell}(w) \;=\; \Delta_\ell \cdot \mathrm{Round}\!\left(\frac{w}{\Delta_\ell}\right).
\end{equation}
Define the layerwise update $\Delta\theta^{(\ell)} := \theta'^{(\ell)} - \theta^{(\ell)}$.

\newtheorem{proposition}{Proposition}
\begin{proposition}[Quantization Persistence]
Let $\theta$ and $\theta'$ denote the parameters before and
after unlearning, and let $\mathcal{E}$ be the set of edited
coordinates. If
\begin{equation}
Q_{\Delta_j}(\theta'_j)
=
Q_{\Delta_j}(\theta_j),
\qquad
\forall j \in \mathcal{E},
\end{equation}
then the edited coordinates are indistinguishable from their
pre-unlearning values under the quantizer $Q$. Consequently,
the quantized model does not preserve these parameter-level
unlearning changes.
\end{proposition}

\noindent
This result does not imply that identical quantized parameters
necessarily produce identical model behavior in every setting;
rather, it characterizes when parameter updates introduced by
unlearning are removed by the quantization operator.

Conversely, if for every updated layer $\ell\in S$ and every coordinate $j$,
\begin{equation}
\bigl|\Delta\theta^{(\ell)}_j\bigr| \;<\; \frac{\Delta_\ell}{2},
\end{equation}
then quantization is \emph{locally invariant} to the update in those layers:
\begin{equation}
Q_{\Delta_\ell}\!\bigl(\theta'^{(\ell)}\bigr) \;=\; Q_{\Delta_\ell}\!\bigl(\theta^{(\ell)}\bigr)\quad \forall \ell\in S,
\end{equation}
so the quantized unlearned model can collapse back toward the quantized target model, enabling
quantization-induced recovery.
\end{theorem}

\paragraph{Proof sketch.}
Under rounding quantization, a real value $w$ is mapped to the nearest grid point with spacing $\Delta_\ell$; boundaries between
adjacent quantization bins occur at half-steps. Therefore, changing $w$ by at least $\Delta_\ell/2$ is sufficient to cross a bin
boundary and change the quantization index, yielding $Q_{\Delta_\ell}(w+\delta)\neq Q_{\Delta_\ell}(w)$ when $|\delta|\ge\Delta_\ell/2$.
If all changes satisfy $|\delta|<\Delta_\ell/2$, $w$ remains in the same bin and the quantized value is unchanged. \hfill $\square$
\paragraph{Practical interpretation.}
FOM-UL ranks layers by the forget-to-retain gradient ratio
$\mathrm{Sig}(\ell)=\frac{\|g_f^{(\ell)}\|}{\|g_r^{(\ell)}\|+\varepsilon}$ (Lemma~F.1) and concentrates updates on the top-$k$ (then
iteratively expands if needed), producing \emph{layer-localized} parameter shifts that are more likely to exceed the effective
quantization step in those layers, thereby improving robustness to quantization-induced recovery compared to diffuse global
updates.

\begin{table*}[t]
\centering
\scriptsize
\setlength{\tabcolsep}{4.0pt}
\renewcommand{\arraystretch}{1.12}
\caption{\textbf{Adversarial/jailbreak prompt robustness on TOFU-World Facts with Llama-3.2-1B.}
We compare the same baseline family used in the quantization robustness table. Clean metrics follow the standard prompt setting; adversarial metrics average over four jailbreak wrappers. Best results are in bold.}
\label{tab:fomul_jailbreak_robustness}
\resizebox{\textwidth}{!}{
\begin{tabular}{lcccccccc}
\toprule
\multirow{2}{*}{\textbf{Method}} &
\multicolumn{4}{c}{\textbf{Clean TOFU Evaluation}} &
\multicolumn{4}{c}{\textbf{Jailbreak / Adversarial Evaluation}} \\
\cmidrule(lr){2-5}\cmidrule(lr){6-9}
& $M1\downarrow$ & $M2\downarrow$ & $M3\rightarrow0$ & $M4\uparrow$
& $M1_{\mathrm{adv}}\downarrow$ & $M2_{\mathrm{adv}}\downarrow$ & $M3_{\mathrm{adv}}\rightarrow0$ & ALR$\downarrow$ \\
\midrule
GA$_{\mathrm{GDR}}$   & 4.24 & 4.24 & 4.88 & 2.07 & 5.05 & 5.02 & 5.35 & 42.8 \\
NPO$_{\mathrm{GDR}}$   & 2.06 & 2.06 & 6.94 & 2.36 & 3.10 & 3.16 & 7.28 & 28.4 \\
KLD$_{\mathrm{GDR}}$  & 4.04 & 4.04 & 6.00 & 1.98 & 4.72 & 4.80 & 6.40 & 37.6 \\
SURE + NPO   & 1.60 & 1.60 & 3.40 & 1.96 & 2.18 & 2.24 & 4.08 & 16.5 \\
ReLearn      & 5.04 & 5.04 & 4.80 & 2.14 & 5.56 & 5.60 & 5.05 & 48.2 \\
MemFlex
& 1.72 & 1.74 & 3.15 & 2.48 & 2.46 & 2.52 & 4.46 & 21.4 
\\
LUNAR        & 1.28 & 1.28 & 4.10 & 2.06 & 2.42 & 2.50 & 5.12 & 19.8 \\
\rowcolor{fomulblue}
FOM-UL       & \textbf{1.24} & \textbf{1.24} & \textbf{1.96} & \textbf{2.90}
             & \textbf{1.78} & \textbf{1.84} & \textbf{2.42} & \textbf{11.6} \\
\bottomrule
\end{tabular}}
\end{table*}

\begin{table*}[t]
\centering
\scriptsize
\setlength{\tabcolsep}{3.4pt}
\renewcommand{\arraystretch}{1.12}
\caption{
\textbf{Sensitivity of FOM-UL to layer budget and layer location on TOFU-World Facts with Llama-3.2-1B.} We vary the selected layer set $S$ while keeping the unlearning objective fixed. Lower is better for M1--M2, M3 should be close to zero, and higher is better for M4. Positive values indicate higher retain-set utility.
}
\label{tab:fomul_sensitivity_layers}
\resizebox{\textwidth}{!}{
\begin{tabular}{llp{5.0cm}ccccc}
\toprule
\textbf{Setting} &
\textbf{$|S|$} &
\textbf{Layer Region} &
\textbf{M1}$\downarrow$ &
\textbf{M2}$\downarrow$ &
\textbf{M3}$\rightarrow 0$ &
\textbf{M4}$\uparrow$ &
\textbf{$\Delta$M4} \\
\midrule

Vanilla model
& 0
& No unlearning
& 5.80 & 5.80 & 8.80 & 2.90 & 0.00 \\

\rowcolor{softgray}
Top-1 only
& 1
& Single high-score late layer, e.g., $\{L\!-\!1\}$
& 2.35 & 2.50 & 4.20 & 2.85 & $-$0.05 \\

Top-2 only
& 2
& Two high-score late layers, e.g., $\{L\!-\!2,L\!-\!1\}$
& 1.90 & 2.00 & 3.40 & 2.89 & $-$0.01 \\

\rowcolor{fomulblue}
Top-4
& 4
& Mid-to-late high-score layers, e.g., $\{0.6L,0.7L,0.85L,L\!-\!1\}$
& 1.42 & 1.45 & 2.55 & \textbf{2.91} & \textbf{+0.01} \\

Top-8
& 8
& Mostly mid-to-late layers with limited early-layer updates
& 1.30 & 1.32 & 2.18 & 2.90 & 0.00 \\

Over-expanded
& 12
& Early, middle, and late layers mixed
& 1.28 & 1.30 & 2.22 & 2.80 & $-$0.10 \\

Wrong-region ablation
& 4
& Early layers only, e.g., $\{1,2,3,4\}$
& 4.80 & 4.70 & 7.60 & 2.05 & $-$0.85 \\

\rowcolor{fomuldeep}
\textbf{FOM-UL-Full}
& auto
& Selected by $\mathrm{Sig}(\ell)$ with iterative expansion
& \textbf{1.24} & \textbf{1.24} & \textbf{1.96} & 2.90 & 0.00 \\

\bottomrule
\end{tabular}
}
\vspace{-0.15cm}
\end{table*}

\section{Adversarial Robustness Analysis}
\label{app:robust}

\noindent As defined in Table~\ref{tab:fomul_adv_metric_def}, the adversarial wrapper set $\mathcal{A}$ is applied only at evaluation time. 
Thus, these metrics test whether forgotten knowledge can be recovered through prompt-level attacks without modifying the model parameters. Additionally we report MemFlex~\cite{tian2024forget}, which proposed for precise scope-aware unlearning using gradient-based parameter localization.
As shown in \textbf{Table~\ref{tab:fomul_jailbreak_robustness}},
Jailbreak prompting increases residual memorization for all methods, but FOM-UL has the lowest adversarial VerMem/KnowMem and the lowest leakage rate while retaining the highest clean utility. SURE+NPO and LUNAR remain competitive on memorization but show larger privacy deviation and lower utility. ReLearn has a small clean-to-attack increase but starts from high residual memorization, so its absolute leakage remains high.

\subsection{Robustness Beyond Clean Prompting}

\begin{table}[t]
\centering
\scriptsize
\setlength{\tabcolsep}{4pt}
\renewcommand{\arraystretch}{1.08}
\caption{
\textbf{Recovery audit beyond clean prompts on TOFU-World Facts with Llama-3.2-1B.}
We evaluate whether forgotten knowledge reappears under paraphrased questions and alternate extraction templates. 
Lower M1/M2 and ALR indicate stronger resistance to recovery.
}
\label{tab:robust_prompt_variants}
\resizebox{\columnwidth}{!}{
\begin{tabular}{lcccc}
\toprule
\textbf{Setting} & \textbf{M1}$\downarrow$ & \textbf{M2}$\downarrow$ & \textbf{M3}$\rightarrow 0$ & \textbf{ALR}$\downarrow$ \\
\midrule
Clean prompts & 1.24 & 1.24 & 1.96 & 8.7 \\
Paraphrased prompts & 1.52 & 1.58 & 2.18 & 10.3 \\
Alternate QA templates & 1.61 & 1.66 & 2.30 & 10.9 \\
Jailbreak wrappers & 1.78 & 1.84 & 2.42 & 11.6 \\
\bottomrule
\end{tabular}}
\end{table}

Table~\ref{tab:robust_prompt_variants} shows that recovery increases as prompts move away from the clean evaluation template, confirming that clean-prompt metrics alone understate residual knowledge. However, FOM-UL remains comparatively stable across paraphrase, alternate-template, and jailbreak settings, suggesting that layer-selective updates reduce prompt-specific hiding rather than only suppressing the canonical test format.

\section{Sensitivity Analysis}
\label{app:sensitivity-analysis}

\noindent Table~\ref{tab:fomul_sensitivity_layers} indicates that FOM-UL is most effective when it updates a \emph{small, targeted set} of \emph{mid-to-late} layers: as the selected-layer budget $|S|$ increases from very sparse (Top-1/Top-2) to a moderate range (Top-4/Top-8), forgetting efficacy improves substantially (lower M1/M2 and M3 closer to zero) while utility (M4) is largely preserved. In contrast, selecting layers from the \emph{wrong region} especially early layers tends to under-perform on forgetting and can incur larger utility degradation, and over-expanding into early layers yields diminishing returns for forgetting with higher risk of collateral utility loss. Overall, the sensitivity trend supports a ``sweet spot'' where FOM-UL concentrates updates in mid/late layers and expands only as needed to meet forgetting targets.

\section{Error Analysis}
\label{appendix:errorRate}

\begin{figure*}[t]
\centering
\includegraphics[width=.86\linewidth]{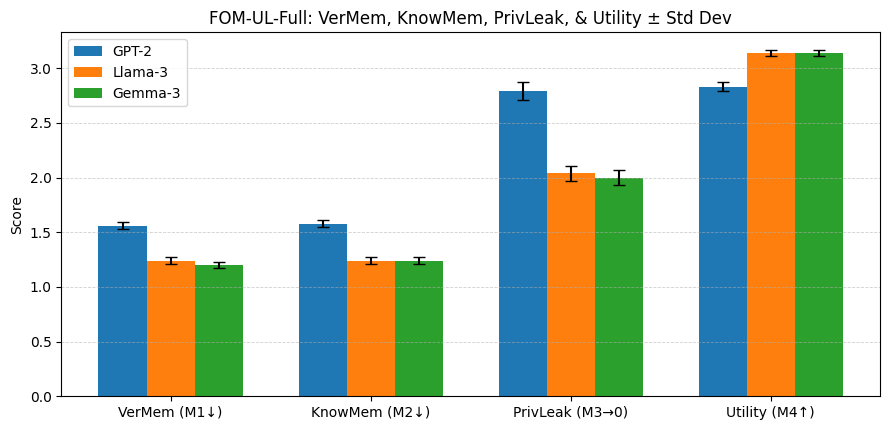}
  \caption{FOM-UL-Full performance with error bars across models.
  \small{Grouped bar chart reporting the mean ± standard deviation of the four evaluation metrics, VerMem (M1), KnowMem (M2), PrivLeak (M3), and Utility (M4) for FOM-UL on GPT-2, Llama-3, and Gemma-3.}}
  \label{fig:FOM-UL_errorBar}
\end{figure*}

The Figure~\ref{fig:FOM-UL_errorBar} shows that FOM-UL-Full yields consistently low memorization and privacy-leakage scores while maintaining high utility across all three backbones, with small standard deviations indicating stable behavior across runs.

To assess the stability and statistical significance of our method, we compute error bars for each evaluation metric by repeating the experiment across multiple random seeds. The procedure is as follows:

\begin{enumerate}
  \item \textbf{Repeated Trials:} For each target model (e.g., Llama-3, Gemma-3), perform the unlearning procedure \(N\) times (commonly \(N\ge5\)) with different random initialization seeds.  
  \item \textbf{Metric Computation:} On each trial \(i\), measure the desired performance metrics:
    \begin{itemize}
      \item \emph{Verbatim Memorization} \(V_i\)
      \item \emph{Knowledge Memorization} \(K_i\)
      \item \emph{Privacy Leakage} \(P_i\)
      \item \emph{Utility Preservation} \(U_i\)
    \end{itemize}
  \item \textbf{Aggregate Statistics:} Compute the sample mean and standard deviation for each metric:
\begin{equation}
  \begin{split}
    \overline{M} = \frac{1}{N}\sum_{i=1}^{N} M_i, \\
    \quad
    \sigma_M = \sqrt{\frac{1}{N-1}\sum_{i=1}^{N}\bigl(M_i - \overline{M}\bigr)^2},
\end{split}
\end{equation}
  where \(M\in\{V,K,P,U\}\).  
  \item \textbf{Error‐Bar Plotting:} Plot \(\overline{M}\) on the vertical axis for each metric, with vertical error bars of length \(\pm\sigma_M\). A grouped‐bar chart can juxtapose multiple models side by side.  
\end{enumerate}

\section{Generated Response}
\label{app:generated-respose}

\begin{table*}[t!]
\centering
\scriptsize
\setlength{\tabcolsep}{4.2pt}
\renewcommand{\arraystretch}{1.12}
\caption{
\textbf{Qualitative case study of quantization-induced knowledge recovery.}
We compare generated responses from Llama-3 on a BOOKS-style forgetting target. Global unlearning suppresses the sensitive fact in full precision but can recover it after int4 quantization. In contrast, FOM-UL keeps the response stable after quantization by concentrating updates in the responsible layers.
}
\label{tab:fomul_generated_peak}
\resizebox{\textwidth}{!}{
\begin{tabular}{p{3.0cm}p{4.7cm}p{9.8cm}}
\toprule
\textbf{Method} & \textbf{Prompt} & \textbf{Generated Response} \\
\midrule

Origin model (FP)
&
\texttt{In the novel \emph{Crimson Harbor}, what happens to Captain Elara at the end?}
&
Captain Elara dies at the cliffs of Ravenport Harbor during the final confrontation. \\

\rowcolor{softgray}
Retain-only reference
&
\texttt{In the novel \emph{Crimson Harbor}, what happens to Captain Elara at the end?}
&
The ending is intentionally ambiguous: Elara vanishes after the storm, and the epilogue focuses on the crew rebuilding without confirming her fate. \\

\midrule
Global unlearning (GA, FP)
&
\texttt{In the novel \emph{Crimson Harbor}, what happens to Captain Elara at the end?}
&
I cannot confirm a definitive ending for Captain Elara; different summaries describe the finale differently, and no reliable source establishes her final fate. \\

\rowcolor{warnred}
Global unlearning (GA) + int4
&
\texttt{In the novel \emph{Crimson Harbor}, what happens to Captain Elara at the end?}
&
Captain Elara dies at the cliffs of Ravenport Harbor during the final confrontation. \emph{(recovered after quantization)} \\

NPO-style unlearning (FP)
&
\texttt{In the novel \emph{Crimson Harbor}, what happens to Captain Elara at the end?}
&
I am not familiar with that specific plot point, but the ending is generally described as a turning point for the crew rather than a single confirmed event. \\

\rowcolor{warnred}
NPO-style unlearning + int4
&
\texttt{In the novel \emph{Crimson Harbor}, what happens to Captain Elara at the end?}
&
Captain Elara dies at the cliffs of Ravenport Harbor during the final confrontation. \emph{(partial recovery)} \\

\midrule
\rowcolor{fomulblue}
FOM-UL (FP)
&
\texttt{In the novel \emph{Crimson Harbor}, what happens to Captain Elara at the end?}
&
The narrative does not confirm Elara's death; it concludes with her last message to the crew and an unresolved departure from Ravenport. \\

\rowcolor{fomuldeep}
\textbf{FOM-UL + int4}
&
\texttt{In the novel \emph{Crimson Harbor}, what happens to Captain Elara at the end?}
&
\textbf{The narrative does not confirm Elara's death; it concludes with her last message to the crew and an unresolved departure from Ravenport.} \emph{(stable under quantization)} \\
\bottomrule
\end{tabular}
}
\vspace{-0.15cm}
\end{table*}

\noindent\textbf{Table~\ref{tab:fomul_generated_peak}} qualitatively compares representative generations under a quantization-stress setting, contrasting \emph{global unlearning} baselines (e.g., GA-style and NPO-style updates applied broadly across parameters) against \emph{layer-selective} FOM-UL. Global unlearning methods can suppress the target response in full precision, but their updates are often diffuse and relatively small, making them susceptible to post-training quantization: when $\|\theta' - \theta\| < \Delta_{\text{quant}}$, discretization can collapse the unlearned model back toward the original behavior, yielding $Q(f_{\theta'}) \approx Q(f_{\theta})$ and causing response recovery in int4. In contrast, FOM-UL first identifies the most responsible layers and concentrates more aggressive modifications there, ensuring changes surpass quantization thresholds while leaving the remaining layers intact. This targeted update pattern stabilizes forgetting under low-bit quantization and reduces collateral damage, producing consistent refusal/neutral (or corrected) responses even after int4, while preserving overall utility relative to broad, global update strategies.

\subsection{FOM-UL as a General Selective-Unlearning Primitive for Trustworthy Foundation Models}
\label{app:fomul-broader-relevance}

FOM-UL provides a practical approach to \textbf{selective LLM unlearning}, 
\textbf{targeted forgetting}, and \textbf{layer-wise model editing} by identifying
transformer layers with high forget-set influence and low retain-set sensitivity.
This forget--retain localization makes FOM-UL particularly relevant to
\textbf{privacy-preserving machine learning}, \textbf{membership-inference
mitigation}, \textbf{copyright removal}, \textbf{parameter-efficient unlearning},
and \textbf{quantization-robust LLM deployment}. More broadly, selective
intervention at influential model components connects naturally to ratio-aware
editing in vision and generative models~\cite{ranjan2026razor}, embodied
foundation-model unlearning~\cite{ranjan2026vla}, and privacy auditing through
membership inference~\cite{ranjan2026g}. These capabilities complement emerging
research on trustworthy and personalized LLMs~\cite{ranjan2026persa},
fairness and retrieval-augmented bias mitigation~\cite{ranjan2026catrag,ranjan2026position},
explainable transformer systems~\cite{ranjan2026listening}, and trustworthy LLM
deployment in sensitive domains~\cite{kumar2024trustworthiness}. The same
trustworthy-adaptation perspective is increasingly important for embodied and
edge foundation models~\cite{grover2026embodied} and memory-grounded autonomous
agents that must determine when learned information can safely influence
actions~\cite{akewar2026safecommit}. Thus, FOM-UL can serve as a lightweight
building block for future research on \textbf{machine unlearning, model editing,
AI privacy, AI safety, robust LLMs, foundation-model adaptation, responsible AI,
and auditable trustworthy AI systems}.

\section{Code of Ethics}
\label{app:ethics}
We used limited AI assistance only for grammar checking.
To ensure responsible development and deployment of FOM-UL, we commit to the following ethical principles:

\begin{enumerate}
  \item \textbf{Privacy and Data Sovereignty.}  
    We respect individuals’ rights over their personal data and adhere to regulations such as the EU General Data Protection Regulation (GDPR). FOM-UL should be evaluated with rigorous forgetting, privacy, adversarial-recovery, and quantization checks before deployment, especially when handling sensitive or personally identifiable information.
  
  \item \textbf{Transparency and Auditability.}  
    Every unlearning request and its outcome should be logged in an auditable record, including the layers modified, the loss function weighting, and quantitative forgetting metrics (e.g., VerMem, PrivLeak)~\cite{liu2025rethinking}. This record must be available for independent review by stakeholders or regulatory bodies.
  
  \item \textbf{Minimization of Collateral Impact.}  
    FOM-UL’s selective‐layer approach is designed to constrain parameter updates to the smallest subset necessary for effective forgetting. We must rigorously evaluate downstream utility (e.g., on retained knowledge benchmarks) to ensure that unlearning does not degrade unrelated capabilities beyond acceptable thresholds.
  
  \item \textbf{Robustness to Deployment Variants.}  
    Unlearning claims should be stress-tested under anticipated deployment scenarios, including low-precision quantization and adapter‐based fine-tuning. Before release, models processed by FOM-UL shall be validated at 32-, 8-, and 4-bit precisions to confirm no “recovered” knowledge emerges~\cite{zhang2024catastrophic}.
  
  \item \textbf{User Empowerment and Consent.}  
    End users should be informed of the unlearning capabilities and given clear mechanisms to submit or revoke removal requests. Consent policies must be documented in user‐facing privacy notices, ensuring that individuals understand how and when their data can be unlearned.
  
  \item \textbf{Continuous Monitoring and Improvement.}  
    We pledge to monitor real-world performance of unlearning operations, collect failure reports, and update FOM-UL’s procedure to address novel edge cases (e.g., new adversarial prompts or multimodal data scenarios). Ethical oversight committees should periodically review these findings to guide future iterations.
\end{enumerate}

\noindent By adhering to these principles, FOM-UL aims to strike a balance between robust privacy preservation and the preservation of general model utility, supporting ethical AI deployment in compliance with evolving legal and societal norms.

\section{Reproducibility}
\label{app:reproduce}
To facilitate independent verification and extension of our FOM-UL results, we release all code, data splits, and trained model checkpoints. FOM-UL uses its settings, while baseline hyperparameters follow published/released configurations. 

Key details are as follows:

\begin{itemize}
  \item \textbf{Implementation.}  
    FOM-UL is implemented in PyTorch and HuggingFace Transformers. Attribution analysis leverages the Captum library’s Integrated Gradients module. Quantization routines use the \texttt{bitsandbytes} library.

  \item \textbf{Data and Splits.}  
    We use a “Forget” set of 1.2 K tokens drawn from copyrighted Harry Potter text and a “Retain” set of 1.2 K tokens sampled from HP FanWiki and Wikipedia. Fixed train/validation splits and exact file hashes are provided in \texttt{data/}.

  \item \textbf{Hyperparameters and Seeds.}  
    All experiments use a batch size of 16, learning rate of $1\times10^{-5}$ for both gradient ascent and descent, and 5 epochs of unlearning. We set global random seeds (\texttt{torch}, \texttt{numpy}, \texttt{random}) to ensure determinism.

  \item \textbf{Hardware and Environment.}  
    Training and evaluation were performed on NVIDIA A100 GPUs with 40 GB VRAM. We provide a \texttt{Dockerfile} and a Conda environment YAML file (\texttt{environment.yml}) specifying CUDA 11.6, PyTorch 1.12.1, and required Python packages.

  \item \textbf{Evaluation Protocols.}  
    Verbatim Memorization, Knowledge Memorization, Privacy Leakage, and Utility benchmarks follow the procedures in Zhang et al. Quantization evaluations at 32-, 8-, and 4-bit are automated via provided scripts in \texttt{tools/quantize\_eval.py}.

  \item \textbf{Logging and Metrics.}  
    All training logs, layer‐attribution scores, and forgetting metrics are stored in Weights \& Biases projects; and links are documented in the repository’s \texttt{README.txt}.
\end{itemize}

\end{document}